\documentclass{article}

\PassOptionsToPackage{numbers, compress}{natbib}
\usepackage[final]{nips_2017} %

\usepackage[utf8]{inputenc} %
\usepackage[T1]{fontenc}    %
\usepackage{hyperref}       %
\usepackage{url}            %
\usepackage{booktabs}       %
\usepackage{amsfonts}       %
\usepackage{nicefrac}       %
\usepackage{microtype}      %

\usepackage{times}
\usepackage{graphicx} %
\usepackage{algorithm}
\usepackage{algorithmic}

\usepackage{amsmath, amssymb, amsthm, bm, color}
\usepackage{enumitem}
\usepackage{notation}
\usepackage{times}
\usepackage{subcaption}
\usepackage{booktabs}       %
\usepackage{array}
\usepackage{multirow}
\usepackage[table,xcdraw]{xcolor}
\usepackage{mathtools}

\usepackage{calc} %

\definecolor{mydarkblue}{rgb}{0,0.08,0.45}

\hypersetup{ %
    pdftitle={On Structured Prediction Theory with Calibrated Convex Surrogate Losses},
    pdfauthor={Anton Osokin, Francis Bach, Simon Lacoste-Julien},
    pdfsubject={On Structured Prediction Theory with Calibrated Convex Surrogate Losses},
    pdfkeywords={structured prediction, consistency, surrogate losses, machine learning},
    pdfborder=0 0 0,
    pdfpagemode=UseNone,
    colorlinks=true,
    linkcolor=mydarkblue,
    citecolor=mydarkblue,
    filecolor=mydarkblue,
    urlcolor=mydarkblue,
    pdfview=FitH
}

\begin{document} 
\title{On Structured Prediction Theory with Calibrated Convex Surrogate Losses}

\author{
    Anton Osokin \\
    INRIA/ENS\thanks{DI \'{E}cole normale sup\'{e}rieure, CNRS, PSL Research University}, Paris, France\\
    HSE\thanks{National Research University Higher School of Economics}, Moscow, Russia
    \And
    Francis Bach\\
    INRIA/ENS\footnotemark[1], Paris, France
    \And
    Simon Lacoste-Julien\\
    MILA and DIRO \\
     Universit\'{e} de Montr\'{e}al, Canada
}

\maketitle

\begin{abstract}
We provide novel theoretical insights on structured prediction in the context of \emph{efficient} convex surrogate loss minimization with consistency guarantees.
For any task loss, we construct a convex surrogate that can be optimized via stochastic gradient descent and we prove tight bounds on the so-called ``calibration function'' relating the excess surrogate risk to the actual risk.
In contrast to prior related work, we carefully monitor the effect of the exponential number of classes in the learning guarantees as well as on the optimization complexity.
As an interesting consequence, we formalize the intuition that some task losses make learning harder than others, and that the classical 0-1 loss is ill-suited for structured prediction.
\end{abstract}

\setcounter{footnote}{0}
\section{Introduction}
\label{sec:intro}
Structured prediction is a subfield of machine learning aiming at making multiple interrelated predictions simultaneously. 
The desired outputs (labels) are typically organized in some structured object such as a sequence, a graph, an image, etc.
Tasks of this type appear in many practical domains such as computer vision~\citep{nowozin2011structured}, natural language processing~\citep{smith2011linguistic} and bioinformatics~\citep{durbin1998bio}.

The structured prediction setup has at least two typical properties differentiating it from the classical binary classification problems extensively studied in learning theory:
\begin{enumerate}[topsep=-\parskip,noitemsep,itemindent=0.45cm,leftmargin=0.0cm]
    \item Exponential number of classes: this brings both additional computational and statistical challenges.
    By \emph{exponential}, we mean exponentially large in the size of the natural dimension of output, e.g., the number of all possible sequences is exponential w.r.t. the sequence length.
    \item Cost-sensitive learning: in typical applications, prediction mistakes are \emph{not} all equally costly. The prediction error is usually measured with a highly-structured task-specific loss function, e.g., Hamming distance between sequences of multi-label variables or mean average precision for ranking.
\end{enumerate}

\vspace{\parskip}Despite many algorithmic advances to tackle structured prediction problems~\citep{bakir2007predicting,nowozin2014structpred}, there have been relatively few papers devoted to its theoretical understanding. Notable recent exceptions that made significant progress include~\citet{cortes16} and~\citet{london16pac} (see references therein) which proposed data-dependent generalization error 
bounds in terms of popular empirical convex surrogate losses such as the structured hinge loss~\citep{taskar03,taskar2005learning,tsochantaridis05}.
A question not addressed by these works is whether their algorithms are \emph{consistent}: does minimizing their convex bounds with infinite data lead to the minimization of the task loss as well?
Alternatively, the structured probit and ramp losses are consistent~\cite{mcallester07,mcallester11}, but non-convex and thus it is hard to obtain computational guarantees for them.
In this paper, we aim at getting the property of consistency for surrogate losses that can be \emph{efficiently} minimized with guarantees, and thus we consider \emph{convex} surrogate losses.

The consistency of convex surrogates is well understood in the case of binary classification~\citep{zhang2004annals,bartlett06convexity,steinwart07} and there is significant progress in the case of multi-class 0-1 loss~\citep{zhang04,tewari07} and general multi-class loss functions~\citep{pires2013riskbounds,ramaswamy16calibrDim,williamson16}.
A large body of work specifically focuses on the related tasks of ranking~\cite{duchi10,calauzenes12,ramaswamy13rankSurrogates} and ordinal regression~\citep{pedregosa15ordivalreg}.\vspace{-0.5mm}

\textbf{Contributions.} In this paper, we study consistent convex surrogate losses specifically in the context of an exponential number of classes. 
We argue that even while being consistent, a convex surrogate might not allow efficient learning.
As a concrete example, \citet{ciliberto16} recently proposed a consistent approach to structured prediction, but the constant in their generalization error bound can be exponentially large as we explain in Section~\ref{sec:relatedworks}.
There are two possible sources of difficulties from the optimization perspective: to reach adequate accuracy on the \emph{task} loss, one might need to optimize a surrogate loss to exponentially small accuracy; or to reach adequate accuracy on the \emph{surrogate} loss, one might need an exponential number of algorithm steps because of exponentially large constants in the convergence rate.
We propose a theoretical framework that jointly tackles these two aspects and allows to judge the feasibility of efficient learning.
In particular, we construct a \emph{calibration function}~\cite{steinwart07}, i.e., a function setting the relationship between accuracy on the surrogate and task losses, and normalize it by the means of convergence rate of an optimization algorithm.\vspace{-0.5mm}

Aiming for the simplest possible application of our framework, we propose a family of convex surrogates that are consistent for any given task loss and can be optimized using stochastic gradient descent.
For a special case of our family (quadratic surrogate), we provide a complete analysis including general lower and upper bounds on the calibration function for any task loss, with exact values for the 0-1, block 0-1 and Hamming losses.
We observe that to have a tractable learning algorithm, one needs both a structured loss (not the 0-1 loss) and appropriate constraints on the predictor, e.g., in the form of linear constraints for the score vector functions.
Our framework also indicates that in some cases it might be beneficial to use non-consistent surrogates.
In particular, a non-consistent surrogate might allow optimization only up to specific accuracy, but exponentially faster than a consistent one.\vspace{-0.5mm}

We introduce the structured prediction setting suitable for studying consistency in Sections~\ref{sec:notation} and~\ref{sec:consistency}. We analyze the calibration function for the quadratic surrogate loss in Section~\ref{sec:quadrSurr}. We review the related works in Section~\ref{sec:relatedworks} and conclude in Section~\ref{sec:conclusion}. 

\vspace{-1.5mm}
\section{Structured prediction setup}
\label{sec:notation}
\vspace{-1mm}
In structured prediction, the goal is to predict
a structured output~$\outputvarv \in \outputdomain$
(such as a sequence, a graph, an image) given an input $\inputvarv \in \inputdomain$.
The quality of prediction is measured by a task-dependent \emph{loss function}~$\lossmatrix(\hat{\outputvarv},\outputvarv \mid \inputvarv) \geq 0$ specifying the cost for predicting~$\hat{\outputvarv}$ when the correct output is~$\outputvarv$.
In this paper, we consider the case when the number of possible predictions and the number of possible labels are both finite.
For simplicity,\footnote{Our analysis is generalizable to rectangular losses, e.g., ranking losses studied by~\citet{ramaswamy13rankSurrogates}.
} we also assume that the sets of possible predictions and correct outputs always coincide and do not depend on~$\inputvarv$.
We refer to this set as the set of labels~$\outputdomain$, denote its cardinality by~$\outputvarcard$, and map its elements to $1,\dots,\outputvarcard$.
In this setting, assuming that the loss function depends only on~$\hat{\outputvarv}$ and~$\outputvarv$, but not on~$\inputvarv$ directly, the loss is defined by a loss matrix~$\lossmatrix \in\R^{\outputvarcard \times \outputvarcard}$.
We assume that all the elements of the matrix~$\lossmatrix$ are non-negative and will use~$\Lmax$ to denote the maximal element.
Compared to multi-class classification, $\outputvarcard$ is typically exponentially large in the size of the natural dimension of~$\outputvarv$, e.g., contains all possible sequences of symbols from a finite alphabet.

Following standard practices in structured prediction~\citep{collins2002,taskar03},
we define the prediction model by a \emph{score function} $\scorefunc: \inputdomain \to \R^{\outputvarcard}$ specifying a score $\scorefunc_\outputvarv(\inputvarv)$ for each possible output~$\outputvarv \in \outputdomain$.
The final prediction is done by selecting a label with the maximal value of the score\vspace{-0.5mm}
\begin{equation}
\label{eq:predictor}
\predictor(\scorefunc(\inputvarv)) := \argmax_{\hat{\outputvarv} \in \outputdomain} \scorefunc_{\hat{\outputvarv}}(\inputvarv),
\end{equation}
with some fixed strategy to resolve ties.
To simplify the analysis, we assume that among the labels with maximal scores, the predictor always picks the one with the smallest index.

The goal of prediction-based machine learning consists in finding a predictor that works well on the unseen test set, i.e., data points coming from the same distribution~$\data$ as the one generating the training data.
One way to formalize this is to minimize the generalization error, often referred to as the actual (or population) \emph{risk} based on the loss~$\lossmatrix$,
\begin{equation}
\label{eq:risk}
\risk_\lossmatrix(\scorefunc) := \E_{(\inputvarv, \outputvarv) \sim \data} \;\lossmatrix\bigl( \predictor(\scorefunc(\inputvarv)), \outputvarv \bigr).
\end{equation}
Minimizing the actual risk~\eqref{eq:risk} is usually hard.
The standard approach is to minimize a \emph{surrogate risk}, which is a different objective easier to optimize, e.g., convex.
We define a surrogate loss as a function $\surrogateloss: \R^\outputvarcard \times \outputdomain \to \R$ depending on a score vector $\scorev = \scorefunc(\inputvarv) \in \R^\outputvarcard$ and a target label~$\outputvarv \in \outputdomain$ as input arguments.
We denote the $\outputvarv$-th component of~$\scorev$ with~$\score_\outputvarv$.
The surrogate risk (the $\surrogateloss$-risk) is defined as
\begin{equation}
\label{eq:surrogateRisk}
\risk_\surrogateloss(\scorefunc) := \E_{(\inputvarv, \outputvarv) \sim \data} \;\surrogateloss( \scorefunc(\inputvarv), \outputvarv ),
\end{equation}
where the expectation is taken w.r.t.\ the data-generating distribution~$\data$.
To make the minimization of~\eqref{eq:surrogateRisk} well-defined, we always assume that the surrogate loss~$\surrogateloss$ is bounded from below and continuous.

Examples of common surrogate losses include the structured hinge-loss~\citep{taskar03,tsochantaridis05}
$
\surrogateloss_{\text{SSVM}}( \scorev, \outputvarv )
:=
\max_{\hat{\outputvarv} \in \outputdomain} \bigl( \score_{\hat{\outputvarv}} + \lossmatrix( \hat{\outputvarv}, \outputvarv ) \bigr) - \score_\outputvarv,
$
the log loss (maximum likelihood learning) used, e.g., in conditional random fields~\citep{lafferty01crf},
$
\surrogateloss_{\text{log}}(\scorev, \outputvarv )
:=
 \log (\sum_{\hat{\outputvarv} \in \outputdomain} \exp \score_{\hat{\outputvarv}} ) - \score_\outputvarv,
$
and their hybrids~\citep{pletscher10,gimpel10,hazan10,shi2015hybrid}.

In terms of task losses, we consider the unstructured \emph{0-1 loss} $\lossmatrix_{01}(\hat{\outputvarv}, \outputvarv) := [\hat{\outputvarv} \neq \outputvarv]$,\footnote{\label{footnote:Iverson}Here we use the Iverson bracket notation, i.e., $[A] := 1$ if a logical expression~$A$ is true, and zero otherwise.} and the two following structured losses: \emph{block 0-1 loss} with $\numblocks$~equal blocks of labels $\lossmatrix_{01,\numblocks}(\hat{\outputvarv}, \outputvarv) := [\text{$\hat{\outputvarv}$ and $\outputvarv$ are not in the same block}]$; and (normalized) \emph{Hamming loss} between tuples of $\hamminglen$~binary variables~$\outputvar_\hammingindex$: $\lossmatrix_{\hamming,\hamminglen}(\hat{\outputvarv}, \outputvarv)
:=
\frac{1}{\hamminglen}\sum\nolimits_{\hammingindex=1}^\hamminglen [\hat{\outputvar}_\hammingindex \neq \outputvar_\hammingindex]
$.
To illustrate some aspects of our analysis, we also look at the \emph{mixed loss}~$\lossmatrix_{01,\numblocks,\consbreakpoint}$: a convex combination of the 0-1 and block 0-1 losses, defined as $\lossmatrix_{01,\numblocks,\consbreakpoint} := \consbreakpoint \lossmatrix_{01} + (1 - \consbreakpoint) \lossmatrix_{01,\numblocks}$ for some $\eta \in [0,1]$.

\section{Consistency for structured prediction}
\label{sec:consistency}
\subsection{Calibration function}
We now formalize the connection between the actual risk~$\risk_\lossmatrix$ and the surrogate $\surrogateloss$-risk $\risk_\surrogateloss$ via the so-called \emph{calibration function}, see Definition~\ref{def:calibrationFunc} below~\citep{bartlett06convexity,zhang04,steinwart07,duchi10,pires2013riskbounds}.
As it is standard for this kind of analysis, the setup is \emph{non-parametric}, i.e. it does not take into account the dependency of scores on input variables~$\inputvarv$.
For now, we assume that a family of score functions~$\scorefuncset_\scoresubset$ consists of all vector-valued Borel measurable functions $\scorefunc: \inputdomain \to \scoresubset$ where $\scoresubset \subseteq \R^\outputvarcard$ is a subspace of allowed score vectors, which will play an important role in our analysis.
This setting is equivalent to a pointwise analysis, i.e, looking at the different input $\inputvarv$ independently.
We bring the dependency on the input back into the analysis in Section~\ref{sec:connection} where we assume a specific family of score functions.

Let $\data_{\inputdomain}$ represent the marginal distribution for $\data$ on $\inputvarv$ and $\P(\cdot \mid \inputvarv)$ denote its conditional given $\inputvarv$. 
We can now rewrite the risk~$\risk_\lossmatrix$ and $\surrogateloss$-risk $\risk_\surrogateloss$ as
\begin{equation*}
\risk_\lossmatrix(\scorefunc)
 =
\E_{\inputvarv \sim \data_{\inputdomain}} \; \lossweighted(\scorefunc(\inputvarv), \P(\cdot \mid \inputvarv)),
\quad
\risk_\surrogateloss(\scorefunc)
=
\E_{\inputvarv \sim \data_{\inputdomain}}\;\surrogateweighted(\scorefunc(\inputvarv), \P(\cdot \mid \inputvarv)),
\end{equation*}
where the conditional risk~$\lossweighted$ and the conditional~$\surrogateloss$-risk~$\surrogateweighted$ depend on a vector of scores~$\scorev$ and a conditional distribution on the set of output labels~$\qv$ as
\begin{equation*}
\lossweighted(\scorev, \qv)
:=
\sum\nolimits_{c=1}^\outputvarcard q_c \lossmatrix( \predictor(\scorev), c ),
\quad
\surrogateweighted(\scorev, \qv)
:=
\sum\nolimits_{c=1}^\outputvarcard q_c \surrogateloss( \scorev, c ).
\end{equation*}
The \emph{calibration function}~$\calibrationfunc_{\surrogateloss,\lossmatrix,\scoresubset}$ between the surrogate loss~$\surrogateloss$ and the task loss~$\lossmatrix$ relates the excess surrogate risk with the actual excess risk via the \emph{excess risk bound}:
\begin{equation}
\label{eq:excessRiskBound}
\calibrationfunc_{\surrogateloss,\lossmatrix,\scoresubset}(\excess\lossweighted(\scorev, \qv)) \leq \excess\surrogateweighted(\scorev, \qv), \; \forall \scorev \in \scoresubset, \; \forall \qv \in \simplex_\outputvarcard,
\end{equation}
where 
$
\excess\surrogateweighted(\scorev, \qv)
=
\surrogateweighted(\scorev, \qv) - \inf_{\hat{\scorev} \in \scoresubset} \surrogateweighted(\hat{\scorev}, \qv)
$,
$
\excess\lossweighted(\scorev, \qv)
=
\lossweighted(\scorev, \qv) - \inf_{\hat{\scorev} \in \scoresubset} \lossweighted(\hat{\scorev}, \qv)
$%
\ are the excess risks and $\simplex_\outputvarcard$ denotes the probability simplex on $k$ elements.

In other words, to find a vector $\scorev$ that yields an excess risk smaller than $\eps$, we need to optimize the $\surrogateloss$-risk up to $\calibrationfunc_{\surrogateloss,\lossmatrix,\scoresubset}(\eps)$ accuracy (in the worst case).
We make this statement precise in Theorem~\ref{th:lossConnection} below, and now proceed to the formal definition of the calibration function.

\begin{definition}[Calibration function]
\label{def:calibrationFunc}
For a task loss~$\lossmatrix$, a surrogate loss~$\surrogateloss$, a set of feasible scores~$\scoresubset$, the \emph{calibration function}~$\calibrationfunc_{\surrogateloss,\lossmatrix,\scoresubset}(\eps)$ (defined for $\eps \geq 0$) equals the infimum excess of the conditional surrogate risk when the excess of the conditional actual risk is at least~$\eps$:
\begin{align}
\label{eq:calibrationfunc}
\calibrationfunc_{\surrogateloss,\lossmatrix,\scoresubset}(\eps)
:=
\;\;\;\;
\inf_{\mathclap{\scorev \in \scoresubset, \;\qv \in \simplex_\outputvarcard}} \;&\;\;\;\;\excess\surrogateweighted(\scorev, \qv) \\
\label{eq:calibrationfunc:epsConstr}
\text{\textup{s.t.}} \:&\;\;\;\;\excess\lossweighted(\scorev, \qv) \geq \eps.
\end{align}
We set $\calibrationfunc_{\surrogateloss,\lossmatrix,\scoresubset}(\eps)$ to $+\infty$ when the feasible set is empty.
\end{definition}
By construction, $\calibrationfunc_{\surrogateloss,\lossmatrix,\scoresubset}$ is non-decreasing on $[0, +\infty)$, $\calibrationfunc_{\surrogateloss,\lossmatrix,\scoresubset}(\eps) \geq 0$, the inequality~\eqref{eq:excessRiskBound} holds, and $\calibrationfunc_{\surrogateloss,\lossmatrix,\scoresubset}(0) = 0$.
Note that $\calibrationfunc_{\surrogateloss,\lossmatrix,\scoresubset}$ can be non-convex and even non-continuous (see examples in Figure~\ref{fig:exampleTranferFunctions}). Also, note that large values of $\calibrationfunc_{\surrogateloss,\lossmatrix,\scoresubset}(\eps)$ are better.

\begin{figure}
    \begin{center}
        \begin{tabular}{c@{\qquad\qquad}c}
            \includegraphics[trim = 0mm 0mm 0mm 0mm, clip, width=0.3\columnwidth]{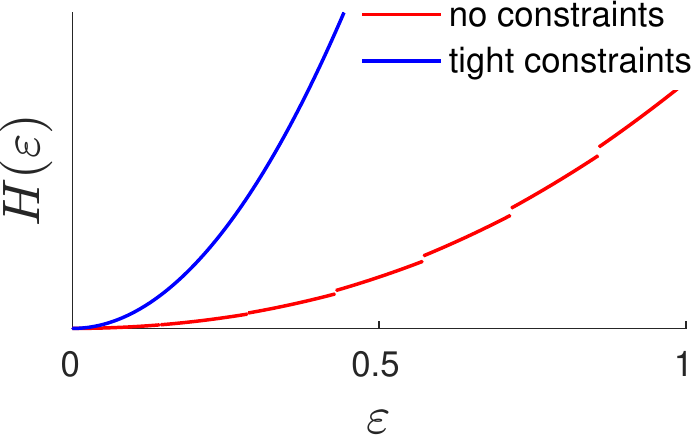} 
            &
            \includegraphics[trim = 0mm 0mm 0mm 0mm, clip, width=0.3\columnwidth]{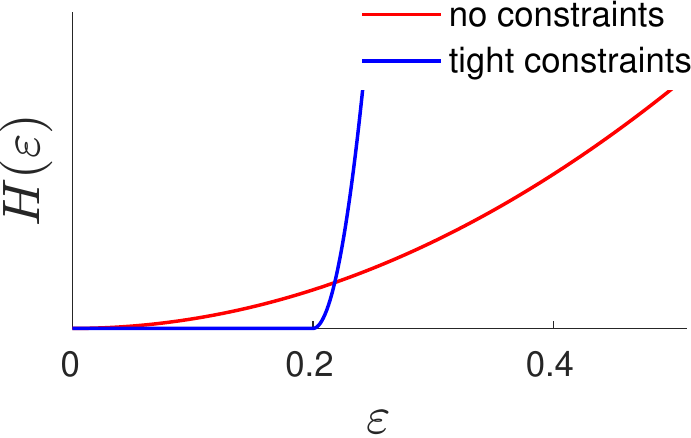} \\
            {\footnotesize (a): Hamming loss $\lossmatrix_{\hamming,\hamminglen}$} & {\footnotesize (b): Mixed loss $\lossmatrix_{01,\numblocks,0.4}$}\\[-1mm]
        \end{tabular}
    \end{center}
    \caption{\label{fig:exampleTranferFunctions} Calibration functions for the quadratic surrogate $\surrogatelossquad$~\eqref{eq:quadrLoss} defined in Section~\ref{sec:quadrSurr} and two different task losses.
        (a)~-- the calibration functions for the Hamming loss $\lossmatrix_{\hamming,\hamminglen}$ when used without constraints on the scores, $\scoresubset = \R^\outputvarcard$ (in red), and with the tight constraints implying consistency, $\scoresubset= \colspace(\lossmatrix_{\hamming,\hamminglen})$ (in blue).
        The red curve can grow exponentially slower than the blue one. 
        (b)~-- the calibration functions for the mixed loss $\lossmatrix_{01,\numblocks,\consbreakpoint}$ with $\consbreakpoint=0.4$ (see Section~\ref{sec:notation} for the definition) when used without constraints on the scores (red) and with tight constraints for the block 0-1 loss (blue).
        The blue curve represents level-$0.2$ consistency.
        The calibration function equals zero for~$\eps \leq \consbreakpoint/2$, but grows exponentially faster than the red curve representing a consistent approach and thus could be better for small $\consbreakpoint$.
        More details on the calibration functions in this figure are given in Section~\ref{sec:quadrSurr}.}
\end{figure}

\subsection{Notion of consistency}
We use the calibration function $\calibrationfunc_{\surrogateloss,\lossmatrix,\scoresubset}$ to set a connection between optimizing the surrogate and task losses by Theorem~\ref{th:lossConnection}, which is similar to Theorem~3 of \citet{zhang04}.
\begin{theorem}[Calibration connection]
\label{th:lossConnection}
Let $\calibrationfunc_{\surrogateloss,\lossmatrix,\scoresubset}$ be the calibration function between the surrogate loss~$\surrogateloss$ and the task loss~$\lossmatrix$ with feasible set of scores~$\scoresubset \subseteq \R^\outputvarcard$. Let $\check{\calibrationfunc}_{\surrogateloss,\lossmatrix,\scoresubset}$ be a convex non-decreasing lower bound of the calibration function.
Assume that~$\surrogateloss$ is continuous and bounded from below.
Then, for any $\eps > 0$ with finite $\check{\calibrationfunc}_{\surrogateloss,\lossmatrix,\scoresubset}(\eps)$ and any $\scorefunc \in \scorefuncset_\scoresubset$, we have
\begin{equation}
\label{eq:theoremConnection:1}
\risk_\surrogateloss(\scorefunc) < \risk_{\surrogateloss,\scoresubset}^* + \check{\calibrationfunc}_{\surrogateloss,\lossmatrix,\scoresubset}(\eps) 
\;\;
\Rightarrow
\;\;
\risk_\lossmatrix(\scorefunc) < \risk_{\lossmatrix,\scoresubset}^* + \eps,
\end{equation}
where $\risk_{\surrogateloss,\scoresubset}^* := \inf_{\scorefunc \in \scorefuncset_\scoresubset} \risk_\surrogateloss(\scorefunc)$ and $\risk_{\lossmatrix,\scoresubset}^*:= \inf_{\scorefunc \in \scorefuncset_\scoresubset} \risk_\lossmatrix (\scorefunc)$. 
\end{theorem}
\begin{proof}
    We take the expectation of~\eqref{eq:excessRiskBound} w.r.t.\ $\inputvarv$, where the second argument of $\lossweighted$ is set to the conditional distribution $\P(\cdot \mid \inputvarv)$.
    Then, we apply Jensen's inequality (since $\check{\calibrationfunc}_{\surrogateloss,\lossmatrix,\scoresubset}$ is convex) to get
    \begin{equation}
    \label{eq:theoremConnection:proof}
    \check{\calibrationfunc}_{\surrogateloss,\lossmatrix,\scoresubset}(\risk_\lossmatrix(\scorefunc) - \risk_{\lossmatrix,\scoresubset}^*)
    \leq
    \risk_\surrogateloss(\scorefunc) - \risk_{\surrogateloss,\scoresubset}^*
    <
    \check{\calibrationfunc}_{\surrogateloss,\lossmatrix,\scoresubset}(\eps),
    \end{equation}
    which implies~\eqref{eq:theoremConnection:1} by monotonicity of $\check{\calibrationfunc}_{\surrogateloss,\lossmatrix,\scoresubset}$.
\end{proof}

A suitable convex non-decreasing lower bound~$\check{\calibrationfunc}_{\surrogateloss,\lossmatrix,\scoresubset}(\eps)$ required by 
Theorem~\ref{th:lossConnection} always exists, e.g., the zero constant. However, in this case Theorem~\ref{th:lossConnection} is not informative, because the l.h.s.\ of~\eqref{eq:theoremConnection:1} is never true.
\citet[Proposition~25]{zhang04} claims that~$\check{\calibrationfunc}_{\surrogateloss,\lossmatrix,\scoresubset}$ defined as the lower convex envelope of the calibration function~$\calibrationfunc_{\surrogateloss,\lossmatrix,\scoresubset}$ satisfies $\check{\calibrationfunc}_{\surrogateloss,\lossmatrix,\scoresubset}(\eps) > 0$, $\forall \eps > 0$, if $\calibrationfunc_{\surrogateloss,\lossmatrix,\scoresubset}(\eps) > 0$, $\forall \eps > 0$, and, e.g., the set of labels is finite.
This statement implies that an informative $\check{\calibrationfunc}_{\surrogateloss,\lossmatrix,\scoresubset}$ always exists and allows to characterize consistency through properties of the calibration function~$\calibrationfunc_{\surrogateloss,\lossmatrix,\scoresubset}$.

We now define a notion of \emph{level-$\consbreakpoint$ consistency}, which is more general than consistency.
\begin{definition}[level-$\consbreakpoint$ consistency]
    \label{def:consistency}
    A surrogate loss~$\surrogateloss$ is \emph{consistent up to level~$\consbreakpoint \geq 0$} w.r.t.\ a task loss~$\lossmatrix$ and a set of scores~$\scoresubset$ if and only if the calibration function satisfies $\calibrationfunc_{\surrogateloss,\lossmatrix,\scoresubset}(\eps) > 0$ for all $\eps > \consbreakpoint$
    and there exists $\hat\eps > \consbreakpoint$ such that $\calibrationfunc_{\surrogateloss,\lossmatrix,\scoresubset}(\hat\eps)$ is finite.
\end{definition}
Looking solely at (standard level-$0$) consistency vs. inconsistency might be too coarse to capture practical properties related to optimization accuracy (see, e.g., \citep{long2013consistency}).
For example, if $\calibrationfunc_{\surrogateloss,\lossmatrix,\scoresubset}(\eps) = 0$ only for very small values of~$\eps$, then the method can still optimize the actual risk up to a certain level which might be good enough in practice, especially if it means that it can be optimized faster.
Examples of calibration functions for consistent and inconsistent surrogate losses are shown in Figure~\ref{fig:exampleTranferFunctions}.

\textbf{Other notions of consistency.}
Definition~\ref{def:consistency} with $\consbreakpoint = 0$ and $\scoresubset = \R^\outputvarcard$ results in the standard setting often appearing in the literature. In particular, in this case Theorem~\ref{th:lossConnection} implies Fisher consistency as formulated, e.g., by~\citet{pedregosa15ordivalreg} for general losses and~\citet{lin04} for binary classification.
This setting is also closely related to many definitions of consistency used in the literature.
For example, for a bounded from below and continuous surrogate, it is equivalent to infinite-sample consistency~\citep{zhang04}, classification calibration~\citep{tewari07}, edge-consistency~\citep{duchi10}, $(\lossmatrix, \R^\outputvarcard)$-calibration~\citep{ramaswamy16calibrDim}, prediction calibration~\citep{williamson16}.
See~\citep[Appendix~A]{zhang04} for the detailed discussion.

\textbf{Role of $\scoresubset$.} Let the \emph{approximation error} for the restricted set of scores $\scoresubset$ be defined as $\risk_{\lossmatrix,\scoresubset}^* - \risk_{\lossmatrix}^* := \inf_{\scorefunc \in \scorefuncset_\scoresubset} \risk_\lossmatrix (\scorefunc) - \inf_{\scorefunc} \risk_\lossmatrix (\scorefunc)$. %
For any conditional distribution~$\qv$, the score vector $\scorev := -L\qv$ will yield an optimal prediction. Thus the condition $\colspace(\lossmatrix) \subseteq \scoresubset$ is sufficient for $\scoresubset$ to have zero approximation error for any distribution~$\data$, and for our $0$-consistency condition to imply the standard Fisher consistency with respect to $\lossmatrix$. In the following, we will see that a restricted $\scoresubset$ can both play a role for computational efficiency as well as statistical efficiency (thus losses with smaller $\colspace(\lossmatrix)$ might be easier to work with).

\subsection{Connection to optimization accuracy and statistical efficiency}
\label{sec:connection}
The scale of a calibration function is not intrinsically well-defined: we could multiply the surrogate function by a scalar and it would multiply the calibration function by the same scalar, without changing the optimization problem.
Intuitively, we would like the surrogate loss to be of order~$1$.
If with this scale the calibration function is exponentially small (has a~$\nicefrac{1}{\outputvarcard}$ factor), then we have strong evidence that the stochastic optimization will be difficult (and thus learning will be slow).

To formalize this intuition, we add to the picture the \emph{complexity} of optimizing the surrogate loss with a \emph{stochastic approximation} algorithm.
By using a scale-invariant convergence rate, we provide a natural normalization of the calibration function.
The following two observations are central to the theoretical insights provided in our work:
\begin{enumerate}[itemindent=0.45cm,leftmargin=0.0cm,topsep=\parsep]
    \item \textbf{Scale.} For a properly scaled surrogate loss, the \emph{scale} of the calibration function is a good indication of whether a stochastic approximation algorithm will take a large number of iterations (in the worst case) to obtain guarantees of small excess of the actual risk (and vice-versa, a large coefficient indicates a small number of iterations). The actual verification requires computing the normalization quantities given in Theorem~\ref{th:calibrationSGD:kernels} below.
    \item \textbf{Statistics.} The bound on the number of iterations directly relates to the number of training examples that would be needed to learn, if we see each iteration of the stochastic approximation algorithm as using one training example to optimize the expected surrogate.
\end{enumerate}

To analyze the statistical convergence of surrogate risk optimization, we have to specify the set of score functions that we work with.
We assume that the structure on input~$\inputvarv \in \inputdomain$ is defined by a positive definite kernel~$\kernelmatrix: \inputdomain \times \inputdomain \to \R$.
We denote the corresponding reproducing kernel Hilbert space (RKHS) by~$\hilbertspace$ and its explicit feature map by~$\featuremap(\inputvarv) \in \hilbertspace$. By the reproducing property, we have $\langle\score, \featuremap(\inputvarv) \rangle_{\hilbertspace} = \score(\inputvarv)$ for all $\inputvarv \in \inputdomain$, $\score \in \hilbertspace$, where $\langle\cdot, \cdot \rangle_{\hilbertspace}$ is the inner product in the RKHS.
We define the subspace of allowed scores~$\scoresubset \subseteq \R^\outputvarcard$ via the span of the columns of a matrix $\scorematrix \in \R^{\outputvarcard \times \scoresubspacedim}$.
The matrix~$\scorematrix$ explicitly defines the structure of the score function.
With this notation, we will assume that the score function is of the form $\scorefunc(\inputvarv) = \scorematrix \parammatrix \featuremap(\inputvarv)$, where $\parammatrix: \hilbertspace \to \R^\scoresubspacedim$ is a linear operator to be learned (a matrix if $\hilbertspace$ is of finite dimension) that represents a collection of~$\scoresubspacedim$ elements in~$\hilbertspace$, transforming $\featuremap(\inputvarv)$ to a vector in~$\R^\scoresubspacedim$ by applying the RKHS inner product~$\scoresubspacedim$ times.\footnote{Note that if $\rank(F) =\scoresubspacedim$, our setup is equivalent to assuming a \emph{joint kernel}~\citep{tsochantaridis05} in the product form: $K_{\text{joint}}((\inputvarv,c),(\inputvarv',c')) := K(\inputvarv,\inputvarv') F(c,:) F(c',:)^\transpose$, where $F(c,:)$ is the row $c$ for matrix $F$.}
Note that for structured losses, we usually have $\scoresubspacedim \ll \outputvarcard$.
The set of all score functions is thus obtained by varying~$\parammatrix$ in this definition and is denoted by~$\scorefuncset_{\scorematrix, \hilbertspace}$.
As a concrete example of a score family~$\scorefuncset_{\scorematrix, \hilbertspace}$ for structured prediction, consider the standard sequence model with unary and pairwise potentials.
In this case, the dimension~$\scoresubspacedim$ equals~$\hamminglen s + (\hamminglen-1) s^2$, where~$\hamminglen$ is the sequence length and $s$ is the number of labels of each variable.
The columns of the matrix~$\scorematrix$ consist of $2\hamminglen - 1$ groups (one for each unary and pairwise potential).
Each row of~$\scorematrix$ has exactly one entry equal to one in each column group (with zeros elsewhere).

In this setting, we use the online projected averaged stochastic subgradient descent ASGD\footnote{See, e.g., \cite{orabona2014simultaneous} for the formal setup of kernel ASGD.} (stochastic w.r.t.\ data $(\inputvarv^{(\iter)}, \outputvarv^{(\iter)}) \sim \data$) to minimize the surrogate risk directly~\cite{bousquet2008tradeoffs}. The $\iter$-th update consists in
\begin{equation}
\label{eq:rkhsSgdStep}
\parammatrix^{(\iter)} := \proj_{D} \bigl[\parammatrix^{(\iter-1)} - \gamma^{(\iter)} \scorematrix^\transpose \gradient \surrogateloss \featuremap(\inputvarv^{(\iter)})^\transpose  \bigr],
\end{equation}
where $\scorematrix^\transpose \gradient \surrogateloss \featuremap(\inputvarv^{(\iter)})^\transpose : \hilbertspace \to \R^\scoresubspacedim$ is the stochastic functional gradient, $\gamma^{(\iter)}$ is the step size and $\proj_{D}$ is the projection on the ball of radius~$D$ w.r.t.\ the Hilbert–Schmidt norm\footnote{The Hilbert–Schmidt norm of a linear operator $A$ is defined as~$\|A\|_{HS} = \sqrt{\trace A^\ddagger A }$ where $A^\ddagger$ is the adjoint operator. In the case of finite dimension, the Hilbert–Schmidt norm coincides with the Frobenius matrix norm.}\!\!. The vector $\gradient \surrogateloss \in \R^\outputvarcard$ is a regular gradient of the sampled surrogate $\surrogateloss( \scorefunc(\inputvarv^{(\iter)}), \outputvarv^{(\iter)} )$ w.r.t.\ the scores, $\gradient \surrogateloss = \gradient_\scorev \surrogateloss( \scorev, \outputvarv^{(\iter)} )|_{\scorev = \scorefunc(\inputvarv^{(\iter)})}$.
We wrote the above update using an explicit feature map $\featuremap$ for notational simplicity, but kernel ASGD can also be implemented without it by using the kernel trick.
The convergence properties of ASGD in RKHS are analogous to the finite-dimensional ASGD because they rely on dimension-free quantities.
To use a simple convergence analysis, we follow~\citet{ciliberto16} and make the following simplifying assumption:
\begin{assumption}[Well-specified optimization w.r.t. the function class~$\scorefuncset_{\scorematrix, \hilbertspace}$]
    \label{th:well_specification}
    The distribution~$\data$ is such that $\risk_{\surrogateloss,\scoresubset}^* := \inf_{\scorefunc \in \scorefuncset_\scoresubset} \risk_\surrogateloss(\scorefunc)$ has some global minimum~$\scorefunc^*$ that also belongs to~$\scorefuncset_{\scorematrix, \hilbertspace}$.
\end{assumption}
Assumption~\ref{th:well_specification} simply means that each row of
$\parammatrix^*$ defining~$\scorefunc^*$ belongs to the RKHS~$\hilbertspace$ implying a finite norm $\| \parammatrix^* \|_{HS}$.
Assumption~\ref{th:well_specification} can be relaxed if the kernel~$K$ is universal, but then the convergence analysis becomes much more complicated~\cite{orabona2014simultaneous}.

\begin{theorem}[Convergence rate]
\label{th:rkhsSgdConvergence}
Under Assumption~\ref{th:well_specification} and assuming that (i) the functions $\surrogateloss( \scorev, \outputvarv )$ are  bounded from below and convex w.r.t.\ $\scorev \in \R^{\outputvarcard}$ for all $\outputvarv \in \outputdomain$; (ii) the expected square of the norm of the stochastic gradient is bounded, $\E_{(\inputvarv, \outputvarv) \sim \data} \| \scorematrix^\transpose \gradient \surrogateloss \featuremap(\inputvarv)^\transpose\|_{HS}^2 \leq M^2$ and (iii) $\| \parammatrix^* \|_{HS} \leq D$, then running the ASGD algorithm~\eqref{eq:rkhsSgdStep} with the constant step-size~$\gamma := \frac{2D}{M \sqrt{\maxiter}}$ for $\maxiter$~steps admits the following expected suboptimality for the averaged iterate $\bar{\scorefunc}^{(\maxiter)}$:
\begin{equation}
\label{eq:rkhsSgdRate}
\E[\risk_\surrogateloss(\bar{\scorefunc}^{(\maxiter)})] -  \risk_{\surrogateloss,\scoresubset}^*
\leq
\frac{2D M}{\sqrt{\maxiter}}
\quad
\text{where}\;\;
\bar{\scorefunc}^{(\maxiter)} := \frac{1}{\maxiter} \sum\nolimits_{\iter=1}^{\maxiter} \!\scorematrix\parammatrix^{(\iter)} \featuremap(\inputvarv^{(\iter)})^\transpose.
\end{equation}
\end{theorem}
Theorem~\ref{th:rkhsSgdConvergence} is a straight-forward extension of classical results~\cite{nemirovski09,orabona2014simultaneous}.

By combining the convergence rate of Theorem~\ref{th:rkhsSgdConvergence} with Theorem~\ref{th:lossConnection} that connects the surrogate and actual risks, we get Theorem~\ref{th:calibrationSGD:kernels} which explicitly gives the number of iterations required to achieve $\eps$~accuracy on the expected \emph{population risk} (see App.~\ref{sec:Thm6Proof} for the proof).
Note that since ASGD is applied in an online fashion, Theorem~\ref{th:calibrationSGD:kernels} also serves as the sample complexity bound, i.e., says how many samples are needed to achieve $\eps$ target accuracy (compared to the best prediction rule if $\scoresubset$ has zero approximation error).
\begin{theorem}[Learning complexity]
    \label{th:calibrationSGD:kernels}
    Under the assumptions of Theorem~\ref{th:rkhsSgdConvergence}, for any~$\eps > 0$, the random (w.r.t.\ the observed training set) output~$\bar{\scorefunc}^{(\maxiter)} \in \scorefuncset_{\scorematrix, \hilbertspace}$ of the ASGD algorithm after
    \begin{equation}
    \label{eq:theoremSgd:2}
    \maxiter > \maxiter^* := \frac{4 D^2 M^2}{\vphantom{(\big(\bigr)}\check{\calibrationfunc}_{\surrogateloss,\lossmatrix,\scoresubset}^2(\eps) }
    \end{equation}
    iterations has the expected excess risk bounded with~$\eps$, i.e.,
    $
    \E[\risk_\lossmatrix(\bar{\scorefunc}^{(\maxiter)})] < \risk_{\lossmatrix,\scoresubset}^* + \eps.
    $
\end{theorem}

\section{Calibration function analysis for quadratic surrogate}
\label{sec:quadrSurr}
A major challenge to applying Theorem~\ref{th:calibrationSGD:kernels} is the computation of the calibration function 
$\calibrationfunc_{\surrogateloss,\lossmatrix,\scoresubset}$.
In App.~\ref{sec:specialLosses}, we present a generalization to arbitrary multi-class losses of a surrogate loss class from~\citet[Section 4.4.2]{zhang04} that is consistent for any task loss~$\lossmatrix$. Here, we consider the simplest example of this family, called the \emph{quadratic surrogate}~$\surrogatelossquad$, which has the advantage that we can bound or even compute exactly its calibration function.
We define the quadratic surrogate as
\begin{equation}
\label{eq:quadrLoss}
\surrogatelossquad( \scorev, \outputvarv )
:=
\frac{1}{2\outputvarcard} \| \scorev + \lossmatrix(:, \outputvarv)\|_2^2 = \frac{1}{2k} \sum_{c=1}^k (f_c^2 + 2f_c L(c,\outputvarv)  + L(c,\outputvarv)^2).
\end{equation}
One simple sufficient condition for the surrogate~\eqref{eq:quadrLoss} to be consistent and also to have zero approximation error is that $\scoresubset$ fully contains~$\colspace(\lossmatrix)$.
To make the dependence on the score subspace explicit, we parameterize it with a matrix~$\scorematrix \in \R^{\outputvarcard \times \scoresubspacedim}$ with the number of columns~$\scoresubspacedim$ typically being much smaller than the number of labels~$\outputvarcard$.
With this notation, we have $\scoresubset = \colspace(\scorematrix) = \{\scorematrix\scoreparamv \mid \scoreparamv\in \R^\scoresubspacedim\}$, and the dimensionality of~$\scoresubset$ equals the rank of~$\scorematrix$, which is at most~$\scoresubspacedim$.\footnote{Evaluating $\surrogatelossquad$ requires computing $\scorematrix^\transpose \scorematrix$ and $\scorematrix^\transpose \lossmatrix(:, \outputvarv)$ for which direct computation is intractable when $\outputvarcard$ is exponential, but which can be done in closed form for the structured losses we consider (the Hamming and block 0-1 loss).
More generally, these operations require suitable inference algorithms. See also App.~\ref{sec:tranferAndSgd}.}

For the quadratic surrogate~\eqref{eq:quadrLoss}, the excess of the expected surrogate takes a simple form:
\begin{equation}
\label{eq:quadSurrogateExcessSimple}
\excess\surrogateweightedquad(\scorematrix\scoreparamv, \qv) = \frac{1}{2\outputvarcard}\| \scorematrix\scoreparamv + \lossmatrix \qv \|_2^2.
\end{equation}
Equation~\eqref{eq:quadSurrogateExcessSimple} holds under the assumption that the subspace~$\scoresubset$ contains the column space of the loss matrix~$\colspace(\lossmatrix)$, which also means that the set $\scoresubset$ contains the optimal prediction for any~$\qv$ (see Lemma~\ref{th:quadSurrogateExcess} in App.~\ref{sec:techLemmas} for the proof).
Importantly, the function~$\excess\surrogateweightedquad(\scorematrix\scoreparamv, \qv)$ is jointly convex in the conditional probability~$\qv$ and parameters~$\scoreparamv$, which simplifies its analysis.

\textbf{Lower bound on the calibration function.}
We now present our main technical result: a lower bound on the calibration function for the surrogate loss $\surrogatelossquad$~\eqref{eq:quadrLoss}. This lower bound characterizes the easiness of learning with this surrogate given the scaling intuition mentioned in Section~\ref{sec:connection}.
The proof of Theorem~\ref{th:lowerBoundcalibrationFunction} is given in App.~\ref{sec:boundscalibrationFunction:lower}.

\begin{theorem}[Lower bound on $\calibrationfunc_{\surrogatelossquad}$]
    \label{th:lowerBoundcalibrationFunction}
    For any task loss~$\lossmatrix$, its quadratic surrogate~$\surrogatelossquad$, and a score subspace~$\scoresubset$ containing the column space of $\lossmatrix$, the calibration function can be lower bounded:
    \begin{equation}
    \label{eq:lowerBoundcalibrationFunction}
    \calibrationfunc_{\surrogatelossquad,\lossmatrix,\scoresubset}(\eps)
    \geq
    \frac{\eps^2}{2\outputvarcard \max_{i\neq j}\|\proj_{\scoresubset} \Delta_{ij}\|_2^2}
    \geq 
    \frac{\eps^2}{4\outputvarcard},
    \end{equation}
    where $\proj_{\scoresubset}$ is the orthogonal projection on the subspace~$\scoresubset$ and $\Delta_{ij} = \unit_i - \unit_j \in \R^\outputvarcard$ with $\unit_c$ being the $c$-th basis vector of the standard basis in~$\R^\outputvarcard$.
\end{theorem}

\textbf{Lower bound for specific losses.} We now discuss the meaning of the bound~\eqref{eq:lowerBoundcalibrationFunction} for some specific losses (the detailed derivations are given in App.~\ref{sec:bounds:interpretations}).
For the 0-1, block 0-1 and Hamming losses ($\lossmatrix_{01}$, $\lossmatrix_{01,\numblocks}$ and $\lossmatrix_{\hamming,\hamminglen}$, respectively) with the smallest possible score subspaces~$\scoresubset$, the bound~\eqref{eq:lowerBoundcalibrationFunction} gives ~$\frac{\eps^2}{4\outputvarcard}$, $\frac{\eps^2}{4\numblocks}$ and $\frac{\epsilon^2}{8\hamminglen}$, respectively.
All these bounds are tight (see App.~\ref{sec:exactTranferProofs}).
However, if~$\scoresubset = \R^\outputvarcard$ the bound~\eqref{eq:lowerBoundcalibrationFunction} is not tight for the block 0-1 and mixed losses (see also App.~\ref{sec:exactTranferProofs}).
In particular, the bound~\eqref{eq:lowerBoundcalibrationFunction} cannot detect level-$\consbreakpoint$ consistency for $\consbreakpoint > 0$ (see Def.~\ref{def:consistency}) and does not change when the loss changes, but the score subspace stays the same.

\textbf{Upper bound on the calibration function.}
Theorem~\ref{th:upperBoundcalibrationFunction} below gives an upper bound on the calibration function holding for unconstrained scores, i.e, $\scoresubset = \R^{\outputvarcard}$ (see the proof in App.~\ref{sec:boundscalibrationFunction:upper}).
This result shows that without some appropriate constraints on the scores, efficient learning is not guaranteed (in the worst case) because of the~$\nicefrac{1}{\outputvarcard}$ scaling of the calibration function.
\begin{theorem}[Upper bound on $\calibrationfunc_{\surrogatelossquad}$]
    \label{th:upperBoundcalibrationFunction}
    If a loss matrix~$\lossmatrix$ with $\Lmax > 0$ defines a pseudometric\footnote{\label{footnote:pseudometric}A pseudometric is a function $d(a,b)$ satisfying the following axioms: $d(x, y) \geq 0$, $d(x, x) = 0$ (but possibly $d(x,y)=0$ for some $x\neq y$), $d(x, y) = d(y, x)$, $d(x, z) \leq d(x, y) + d(y, z)$.} on labels and there are no constraints on the scores, i.e., $\scoresubset = \R^{\outputvarcard}$, then the calibration function for the quadratic surrogate~$\surrogatelossquad$ can be upper bounded:
    $
    \calibrationfunc_{\surrogatelossquad,\lossmatrix,\R^{\outputvarcard}}(\eps)
    \leq  \frac{\eps^2}{2\outputvarcard}, \quad 0 \leq \eps \leq \Lmax.
    $
\end{theorem}
From our lower bound in Theorem~\ref{th:lowerBoundcalibrationFunction} (which guarantees consistency), the natural constraint on the score is $\scoresubset = \colspace(\lossmatrix)$, with the dimension of this space giving an indication of the intrinsic ``difficulty'' of a loss.
Computations for the lower bounds in some specific cases (see App.~\ref{sec:bounds:interpretations} for details) show that the 0-1 loss is ``hard'' while the block 0-1 loss and the Hamming loss are ``easy''.
Note that in all these cases the lower bound~\eqref{eq:lowerBoundcalibrationFunction} is tight, see the discussion below.

\textbf{Exact calibration functions.}
Note that the bounds proven in Theorems~\ref{th:lowerBoundcalibrationFunction} and~\ref{th:upperBoundcalibrationFunction} imply that, in the case of no constraints on the scores~$\scoresubset = \R^\outputvarcard$, for the 0-1, block 0-1 and Hamming losses, we have
\begin{equation}
\label{eq:lowerUpperBounds}
\frac{\eps^2}{4\outputvarcard} \leq \calibrationfunc_{\surrogatelossquad,\lossmatrix,\R^\outputvarcard}(\eps) \leq \frac{\eps^2}{2\outputvarcard},
\end{equation}
where $\lossmatrix$ is the matrix defining a loss.
For completeness, in App.~\ref{sec:exactTranferProofs}, we compute the exact calibration functions for the 0-1 and block 0-1 losses.
Note that the calibration function for the \textbf{0-1 loss} equals the lower bound, illustrating the worst-case scenario. To get some intuition, an example of a conditional distribution~$\qv$ that gives the (worst case) value to the calibration function (for several losses) is $\q_i = \frac{1}{2} + \frac{\eps}{2}$, $\q_j = \frac{1}{2} - \frac{\eps}{2}$ and $\q_c=0$ for $c \not\in \{i,j\}$. See the proof of Proposition~\ref{th:calibrationFunction:01loss} in App.~\ref{sec:exactTranferProofs:01loss}.

In what follows, we provide the calibration functions in the cases with constraints on the scores.
For the \textbf{block 0-1 loss} with~$\numblocks$ equal blocks and under constraints that the scores within blocks are equal, the calibration function equals (see Proposition~\ref{th:calibrationFunction:block01loss:hardConstr} of  App.~\ref{sec:exactTranferProofs:block01loss})
\begin{equation}
\label{eq:calibrationFunction:block01loss:hardConstr}
\calibrationfunc_{\surrogatelossquad,\lossmatrix_{01,\numblocks},\scoresubset_{01,\numblocks}}(\eps)
=
\frac{\eps^2}{4\numblocks},
\quad 0 \leq \eps \leq 1.
\end{equation}

For the \textbf{Hamming loss} defined over $\hamminglen$~binary variables and under constraints implying separable scores, the calibration function equals (see Proposition~\ref{th:calibrationFunction:hammingLoss:hardConstr} in App.~\ref{sec:exactTranferProofs:hamming})
\begin{equation}
\label{eq:calibrationFunction:hammingloss:hardConstr}
\calibrationfunc_{\surrogatelossquad,\lossmatrix_{\hamming,\hamminglen},\scoresubset_{\hamming,\hamminglen}}(\eps)
=
\frac{\eps^2}{8\hamminglen},
\; 0 \leq \eps \leq 1.
\end{equation}

The calibration functions~\eqref{eq:calibrationFunction:block01loss:hardConstr} and~\eqref{eq:calibrationFunction:hammingloss:hardConstr} depend on the quantities representing the actual complexities of the loss (the number of blocks~$\numblocks$ and the length of the sequence~$\hamminglen$) and can be exponentially larger than the upper bound for the unconstrained case.

In the case of \textbf{mixed 0-1 and block 0-1 loss},
if the scores~$\scorev$ are constrained to be equal inside the blocks, i.e., belong to the subspace~$\scoresubset_{01,\numblocks} = \colspace(\lossmatrix_{01,\numblocks}) \subsetneq \R^k$, then the calibration function is equal to~$0$ for $\eps \leq \frac{\consbreakpoint}{2}$, implying inconsistency (and also note that the approximation error can be as big as~$\consbreakpoint$ for $\scoresubset_{01,\numblocks}$).
However, for $\eps > \frac{\consbreakpoint}{2}$, the calibration function is of the order
$
\frac{1}{\numblocks} (\eps-\frac{\consbreakpoint}{2})^2.
$
See Figure~\ref{fig:exampleTranferFunctions}b for the illustration of this calibration function and Proposition~\ref{th:calibrationFunction:mixedLoss:hardConstr} of App.~\ref{sec:exactTranferProofs:mixedloss} for the exact formulation and the proof.
Note that while the calibration function for the constrained case is inconsistent, its value can be exponentially larger than the one for the unconstrained case for $\eps$ big enough and when the blocks are exponentially large (see Proposition~\ref{th:calibrationFunction:mixedLoss} of App.~\ref{sec:exactTranferProofs:mixedloss}).

\textbf{Computation of the SGD constants.}
Applying the learning complexity Theorem~\ref{th:calibrationSGD:kernels} requires to compute the quantity $DM$ where $D$ bounds the norm of the optimal solution and $M$ bounds the expected square of the norm of the stochastic gradient.
In App.~\ref{sec:tranferAndSgd}, we provide a way to bound this quantity for our quadratic surrogate~\eqref{eq:quadrLoss} under the simplifying assumption that each conditional~$\q_c(\inputvarv)$ (seen as function of $\inputvarv$) belongs to the RKHS~$\hilbertspace$ (which implies Assumption~\ref{th:well_specification}).
In particular, we get
\begin{equation}
\label{eq:sgdConstant}
DM
=
\Lmax^2 \xi(\condnum(\scorematrix) \sqrt{\scoresubspacedim} \rkhsbound  \qbound), \quad \xi(z) = z^2 + z,
\end{equation}
where $\condnum(\scorematrix)$ is the condition number of the matrix~$\scorematrix$, $\rkhsbound$ is an upper bound on the RKHS norm of object feature maps $\|\featuremap(\inputvarv)\|_\hilbertspace$. We define $\qbound$ as an upper bound on $\sum_{c=1}^{\outputvarcard} \|\q_c\|_\hilbertspace$ (can be seen as the generalization of the inequality~$\sum_{c=1}^{\outputvarcard} \q_c \leq 1$ for probabilities).
The constants $\rkhsbound$ and $\qbound$ depend on the data, the constant~$\Lmax$ depends on the loss, $\scoresubspacedim$ and $\condnum(\scorematrix)$ depend on the choice of matrix~$\scorematrix$.

We compute the constant $DM$ for the specific losses that we considered in App.~\ref{sec:tranferAndSgd:constants}.
For the 0-1, block 0-1 and Hamming losses, we have $DM = \bigO(\outputvarcard)$, $DM = \bigO(\numblocks)$ and $DM = \bigO(\log_2^3 \outputvarcard)$, respectively.
These computations indicate that the quadratic surrogate allows efficient learning for structured block 0-1 and Hamming losses, but that the convergence could be slow in the worst case for the 0-1 loss.

\section{Related works}
\label{sec:relatedworks}
\textbf{Consistency for multi-class problems.}
Building on significant progress for the case of binary classification, see, e.g.~\cite{bartlett06convexity}, there has been a lot of interest in the multi-class case.
\citet{zhang04} and \citet{tewari07} analyze the consistency of many existing surrogates for the 0-1 loss.
\citet{gao2011multilabel} focus on  multi-label classification.
\citet{narasimhan15} provide a consistent algorithm for arbitrary multi-class loss defined by a function of the confusion matrix.
Recently, \citet{ramaswamy16calibrDim} introduce the notion of convex calibrated dimension, as the minimal dimensionality of the score vector that is required for consistency.
In particular, they showed that for the Hamming loss on~$\hamminglen$ binary variables, this dimension is at most~$\hamminglen$.
In our analysis, we use scores of rank~$(\hamminglen+1)$, see~\eqref{eq:hammingLoss} in App.~\ref{sec:bounds:interpretations}, yielding a similar result.

The task of ranking has attracted a lot of attention and \cite{duchi10,buffoni2011learning,calauzenes12,ramaswamy13rankSurrogates} analyze different families of surrogate and task losses proving their (in-)consistency.
In this line of work, \citet{ramaswamy13rankSurrogates} propose a quadratic surrogate for an arbitrary low rank loss which is related to our quadratic surrogate~\eqref{eq:quadrLoss}.
They also prove that several important ranking losses, i.e.,
precision@q, 
expected rank utility, 
mean average precision and
pairwise disagreement, are of low-rank.
We conjecture that our approach is compatible with these losses and leave precise connections as future work.

\textbf{Structured SVM (SSVM) and friends.}
SSVM~\citep{taskar03,taskar2005learning,tsochantaridis05} is one of the most used convex surrogates for tasks with structured outputs, thus, its consistency has been a question of great interest.
It is known that Crammer-Singer multi-class SVM~\citep{crammer01}, which SSVM is built on, is not consistent for 0-1 loss unless there is a majority class with probability at least $\frac12$~\citep{zhang04,mcallester07}.
However, it is consistent for the ``abstain'' and ordinal losses in the case of~$3$ classes~\cite{ramaswamy16calibrDim}.
Structured ramp loss and probit surrogates are closely related to SSVM and are consistent~\citep{mcallester07,chapelle2009tighter,mcallester11,keshet14}, but not convex.

Recently, \citet{dogan2016svm} categorized different versions of multi-class SVM and analyzed them from Fisher and universal consistency point of views.
In particular, they highlight differences between Fisher and universal consistency and give examples of surrogates that are Fisher consistent, but not universally consistent and vice versa.
They also highlight that the Crammer-Singer SVM is neither Fisher, not universally consistent even with a careful choice of regularizer.

\textbf{Quadratic surrogates for structured prediction.}
\citet{ciliberto16} and \citet{brouard2016input} consider minimizing~$\sum_{i=1}^n \|g(\inputvarv_i) - \featuremap_o(\outputvarv_i)\|_\hilbertspace^2$ aiming to match the RKHS embedding of inputs~$g: \inputdomain \to  \hilbertspace$ to the feature maps of outputs~$\featuremap_o: \outputdomain \to \hilbertspace$.
In their frameworks, the task loss is not considered at the learning stage, but only at the prediction stage.
Our quadratic surrogate~\eqref{eq:quadrLoss} depends on the loss directly.
The empirical risk defined by both their and our objectives can be minimized analytically with the help of the kernel trick and, moreover, the resulting predictors are identical.
However, performing such computation in the case of large dataset can be intractable and the generalization properties have to be taken care of, e.g., by the means of regularization.
In the large-scale scenario, it is more natural to apply stochastic optimization (e.g., kernel ASGD) that directly minimizes the population risk and has better dependency on the dataset size.
When combined with stochastic optimization, the two approaches lead to different behavior.
In our framework, we need to estimate~$\scoresubspacedim = \rank(\lossmatrix)$ scalar functions, but the alternative needs to estimate~$\outputvarcard$ functions (if, e.g., $\featuremap_o(\outputvarv) = \unit_\outputvarv \in \R^\outputvarcard$), which results in significant differences for low-rank losses, such as block 0-1 and Hamming.

\textbf{Calibration functions.}
\citet{bartlett06convexity} and \citet{steinwart07} provide calibration functions for most existing surrogates for binary classification.
All these functions differ in term of shape, but are roughly similar in terms of constants.
\citet{pedregosa15ordivalreg} generalize these results to the case of ordinal regression.
However, their calibration functions have at best a $\nicefrac{1}{\outputvarcard}$ factor if the surrogate is normalized w.r.t.\ the number of classes.
The task of ranking has been of significant interest.
However, most of the literature \citep[e.g.,][]{clemencon2008ranking,cossock2008ranking,kotlowski2011bipartite,agarwal2014regretbounds},
only focuses on calibration functions (in the form of regret bounds) for bipartite ranking, which is more akin to cost-sensitive binary classification.

\citet{pires2013riskbounds} generalize the theoretical framework developed by~\citet{steinwart07}  and present results for the multi-class SVM of~\citet{lee2004multicategory} (the score vectors are constrained to sum to zero) that can be built for any task loss of interest.
Their surrogate $\surrogateloss$ is of the form $\sum_{c \in \outputdomain} \lossmatrix(c, \outputvarv) a( \score_c )$ where $\sum_{c \in \outputdomain} \score_c = 0$ and $a(\score)$ is some convex function with all subgradients at zero being positive.
The recent work by~\citet{pires2016calfuncs} refines the results, but specifically for the case of 0-1 loss.
In this line of work, the surrogate is typically not normalized by~$\outputvarcard$, and if normalized the calibration functions have the constant $\nicefrac{1}{\outputvarcard}$ appearing.

Finally, \citet{ciliberto16} provide the calibration function for their quadratic surrogate.
Assuming that the loss can be represented as
$
\lossmatrix(\hat{\outputvarv}, \outputvarv) = \langle V \featuremap_o(\hat{\outputvarv}), \featuremap_o(\outputvarv) \rangle_{\mathcal{H}_\outputdomain}
$, $\hat{\outputvarv}, \outputvarv \in \outputdomain$
(this assumption can always be satisfied in the case of a finite number of labels, by taking $V$ as the loss matrix~$\lossmatrix$ and $\featuremap_o(\outputvarv) := \unit_\outputvarv \in \R^\outputvarcard$ where $\unit_\outputvarv$ is the $\outputvarv$-th vector of the standard basis in~$\R^{\outputvarcard}$).
In their Theorem~2, they provide an excess risk bound leading to a lower bound on the corresponding calibration function
$
\calibrationfunc_{\surrogateloss,\lossmatrix,\R^{\outputvarcard}}(\eps) \geq \frac{\eps^2}{c_\Delta^2}
$
where a constant $c_\Delta = \|V\|_2 \max_{\outputvarv \in \outputdomain} \|\featuremap_o(\outputvarv)\|$ simply equals the spectral norm of the loss matrix for the finite-dimensional construction provided above.
However, the spectral norm of the loss matrix is exponentially large even for highly structured losses such as the block 0-1 and Hamming losses, i.e.,
$\|\lossmatrix_{01,\numblocks}\|_2 = \outputvarcard - \frac{\outputvarcard}{\numblocks}$,
$\|\lossmatrix_{\hamming,\hamminglen}\|_2 = \frac{\outputvarcard}{2}$.
This conclusion puts the objective of~\citet{ciliberto16} in line with ours when no constraints are put on the scores.

\section{Conclusion}
\label{sec:conclusion}
In this paper, we studied the consistency of convex surrogate losses specifically in the context of structured prediction.
We analyzed calibration functions and proposed an optimization-based normalization aiming to connect consistency with the existence of efficient learning algorithms.
Finally, we instantiated all components of our framework for several losses by computing the calibration functions and the constants coming from the normalization.
By carefully monitoring exponential constants, we highlighted the difference between tractable and intractable task losses.

These were first steps in advancing our theoretical understanding of consistent structured prediction. Further steps include analyzing more losses such as the low-rank ranking losses studied by~\citet{ramaswamy13rankSurrogates} and, instead of considering constraints on the scores, one could instead put constraints on the set of distributions to investigate the effect on the calibration function.  

\subsubsection*{Acknowledgements}
We would like to thank Pascal Germain for useful discussions.
This work was partly supported by the ERC grant Activia (no. 307574), the NSERC Discovery Grant RGPIN-2017-06936 and the MSR-INRIA Joint Center.

\bibliographystyle{icml2017}
{ %
\bibliography{references}}

\clearpage
\appendix

\newcommand{\toptitlebar}{
    \hrule height 4pt
    \vskip 0.25in
    \vskip -\parskip%
}
\newcommand{\bottomtitlebar}{
    \vskip 0.29in
    \vskip -\parskip
    \hrule height 1pt
    \vskip 0.09in%
}
\toptitlebar
\begin{center}
{\LARGE\bf Supplementary Material (Appendix)}
\\[4mm]
{\LARGE\bf On Structured Prediction Theory with Calibrated }
\\
{\LARGE\bf Convex Surrogate Losses }
\end{center}
\bottomtitlebar

\section*{Outline} 
\begin{description}[itemsep=0mm,topsep=0cm,leftmargin=*,font=\normalfont]
    \renewcommand\labelitemi{--}
    \item[Section~\ref{sec:Thm6Proof}:] Proof of learning complexity Theorem~\ref{th:calibrationSGD:kernels}.
    \item[Section~\ref{sec:techLemmas}:] Technical lemmas useful for the proofs.
    \item[Section~\ref{sec:specialLosses}:] Discussion and consistency results on a family of surrogate losses.
    \item[Section~\ref{sec:boundscalibrationFunction}:] Bounds on the calibration functions.
    \begin{description}[itemsep=0mm,topsep=-0.05cm,itemindent=0.0cm,leftmargin=0.6cm,font=\normalfont]
    \item[Section~\ref{sec:boundscalibrationFunction:lower}:] Theorem~\ref{th:lowerBoundcalibrationFunction}~-- a lower bound.
    \item[Section~\ref{sec:boundscalibrationFunction:upper}:]  Theorem~\ref{th:upperBoundcalibrationFunction}~-- an upper bound.
    \item[Section~\ref{sec:bounds:interpretations}:] Computation of the bounds for specific task losses.
    \end{description}
    \item[Section~\ref{sec:exactTranferProofs}:] Computations of the exact calibration functions for the quadratic surrogate.
    \begin{description}[itemsep=0mm,topsep=-0.05cm,itemindent=0.0cm,leftmargin=0.6cm,font=\normalfont]
        \item[Section~\ref{sec:exactTranferProofs:01loss}:] 0-1 loss.
        \item[Section~\ref{sec:exactTranferProofs:block01loss}:] Block 0-1 loss.
        \item[Section~\ref{sec:exactTranferProofs:hamming}:] Hamming loss.
        \item[Section~\ref{sec:exactTranferProofs:mixedloss}:] Mixed 0-1 and block 0-1 loss.
    \end{description}
    \item[Section~\ref{sec:tranferAndSgd}:] Computing constants appearing in the SGD rate.
    \item[Section~\ref{sec:hammingLossProps}:] Properties of the basis of the Hamming loss.
\end{description}

\section{Learning complexity theorem} \label{sec:Thm6Proof}

\begin{reptheorem}{th:calibrationSGD:kernels}[Learning complexity]
    \label{th:rep:calibrationSGD:kernels}
 Under the assumptions of Theorem~\ref{th:rkhsSgdConvergence}, for any~$\eps > 0$, the random (w.r.t.\ the observed training set) output~$\bar{\scorefunc}^{(\maxiter)} \in \scorefuncset_{\scorematrix, \hilbertspace}$ of the ASGD algorithm after
    \begin{equation}
    \maxiter > \maxiter^* := \frac{4 D^2 M^2}{\vphantom{(\big(\bigr)}\check{\calibrationfunc}_{\surrogateloss,\lossmatrix,\scoresubset}^2(\eps) }
    \end{equation}
    iterations has the expected excess risk bounded with~$\eps$, i.e.,
    $
    \E[\risk_\lossmatrix(\bar{\scorefunc}^{(\maxiter)})] < \risk_{\lossmatrix,\scoresubset}^* + \eps.
    $
\end{reptheorem}
\begin{proof}
	By~\eqref{eq:rkhsSgdRate} from Theorem~\ref{th:rkhsSgdConvergence}, $\maxiter$ steps of the algorithm, in expectation, result in~$\check{\calibrationfunc}_{\surrogateloss,\lossmatrix,\scoresubset}(\eps)$ accuracy on the surrogate risk, i.e., $\E[\risk_\surrogateloss(\bar{\scorefunc}^{(\maxiter)})] -  \risk_{\surrogateloss,\scoresubset}^* < \check{\calibrationfunc}_{\surrogateloss,\lossmatrix,\scoresubset}(\eps)$.
    We now generalize the proof of Theorem~\ref{th:lossConnection} to the case of expectation w.r.t.~$\bar{\scorefunc}^{(\maxiter)}$ depending on the random samples used by the ASGD algorithm.
    We take the expectation of~\eqref{eq:excessRiskBound} w.r.t.\ $\bar{\scorefunc}^{(\maxiter)}$ substituted as~$\scorefunc$ and use Jensen's inequality (by convexity of~$\check{\calibrationfunc}_{\surrogateloss,\lossmatrix,\scoresubset}$) to get
    $\E[\risk_\surrogateloss(\bar{\scorefunc}^{(\maxiter)})] -  \risk_{\surrogateloss,\scoresubset}^* \geq \E[\check{\calibrationfunc}_{\surrogateloss,\lossmatrix,\scoresubset}(\risk_\lossmatrix(\bar{\scorefunc}^{(\maxiter)}) - \risk_{\lossmatrix,\scoresubset}^*)] \geq \check{\calibrationfunc}_{\surrogateloss,\lossmatrix,\scoresubset}(\E[\risk_\lossmatrix(\bar{\scorefunc}^{(\maxiter)})] - \risk_{\lossmatrix,\scoresubset}^*)$.
    Finally, monotonicity of~$\check{\calibrationfunc}_{\surrogateloss,\lossmatrix,\scoresubset}$ implies~$\E[\risk_\lossmatrix(\bar{\scorefunc}^{(\maxiter)})] - \risk_{\lossmatrix,\scoresubset}^* < \eps$.
\end{proof}

\section{Technical lemmas}
\label{sec:techLemmas}
In this section, we prove two technical lemmas that simplify the proofs of the main theoretical claims of the paper.

Lemma~\ref{th:quadSurrogateExcess} computes the excess of the weighted surrogate risk~$\excess\surrogateweighted$ for the quadratic loss~$\surrogatelossquad$~\eqref{eq:quadrLoss}, which is central to our analysis presented in Section~\ref{sec:quadrSurr}.
The key property of this result is that the excess~$\excess\surrogateweighted$ is jointly convex w.r.t.\ the parameters~$\scoreparamv$ and conditional distribution~$\qv$, which simplifies further analysis.

Lemma~\ref{th:breakingSymmetries} allows to cope with the combinatorial aspect of the computation of the calibration function.
In particular, when the excess of the weighted surrogate risk is convex, Lemma~\ref{th:breakingSymmetries} reduces the computation of the calibration function to a set of convex optimization problems, which often can be solved analytically.
For symmetric losses, such as the 0-1, block 0-1 and Hamming losses, Lemma~\ref{th:breakingSymmetries} also provides ``symmetry breaking'', meaning that many of the obtained convex optimization problems are identical up to a permutation of labels.
\begin{lemma}
    \label{th:quadSurrogateExcess}
    Consider the quadratic surrogate $\surrogatelossquad$~\eqref{eq:quadrLoss} defined for a task loss~$\lossmatrix$.
    Let a subspace of scores~$\scoresubset \subseteq \R^\outputvarcard$ be parametrized by $\scoreparamv\in \R^\scoresubspacedim$, i.e., $\scorev = \scorematrix\scoreparamv \in \scoresubset$ with $\scorematrix \in \R^{\outputvarcard \times \scoresubspacedim}$, and assume that $\colspace(\lossmatrix) \subseteq \scoresubset$.
    Then, the excess of the weighted surrogate loss can be expressed as
    \begin{equation*}
    \excess\surrogateweightedquad(\scorematrix\scoreparamv, \qv) := 
    \surrogateweightedquad(\scorematrix\scoreparamv, \qv)-\inf_{\scoreparamv' \in \R^\scoresubspacedim} \surrogateweightedquad(\scorematrix\scoreparamv', \qv) =
    \frac{1}{2\outputvarcard}\| \scorematrix\scoreparamv + \lossmatrix \qv \|_2^2.
    \end{equation*}
\end{lemma}
\begin{proof}
    By using the definition of the quadratic surrogate~$\surrogatelossquad$~\eqref{eq:quadrLoss}, we have
    \begin{align*}
    \surrogateweighted(\scorev(\scoreparamv), \qv) &= \frac{1}{2\outputvarcard}(\scoreparamv^\transpose \scorematrix^\transpose \scorematrix \scoreparamv + 2 \scoreparamv^\transpose \scorematrix^\transpose \lossmatrix \qv) + r(\qv), \\
    \scoreparamv^* &:= \argmin\nolimits_\scoreparamv \surrogateweighted(\scorev(\scoreparamv), \qv) =  -(\scorematrix^\transpose \scorematrix)^\pinv \scorematrix^\transpose \lossmatrix \qv, \\
    \excess\surrogateweighted(\scorev(\scoreparamv), \qv) &= \frac{1}{2\outputvarcard}(\scoreparamv^\transpose \scorematrix^\transpose \scorematrix \scoreparamv + 2 \scoreparamv^\transpose \scorematrix^\transpose \lossmatrix \qv \\
    &\qquad\qquad\qquad+ \qv^\transpose \lossmatrix^\transpose \scorematrix (\scorematrix^\transpose \scorematrix)^\pinv \scorematrix^\transpose \lossmatrix \qv),
    \end{align*}
    where $r(\qv)$ denotes the quantity independent of parameters~$\scoreparamv$.
    Note that $\proj_{\scorematrix} := \scorematrix (\scorematrix^\transpose \scorematrix)^\pinv \scorematrix^\transpose$ is the orthogonal projection on the subspace~$\colspace(\scorematrix)$, so if $\colspace(\lossmatrix) \subseteq \colspace(\scorematrix)$ we have 
    $
    \proj_{\scorematrix} \lossmatrix = \lossmatrix,
    $
    which finishes the proof.
\end{proof}

\begin{lemma}
    \label{th:breakingSymmetries}
    In the case of a finite number~$\outputvarcard$ of labels, for any task loss~$\lossmatrix$, a surrogate loss~$\surrogateloss$ that is continuous and bounded from below, and a set of scores~$\scoresubset$, the calibration function can be written as 
    \begin{equation}
    \label{eq:calibrationfunc:decomposition}
    \calibrationfunc_{\surrogateloss,\lossmatrix,\scoresubset}(\eps) = \min_{\substack{i, j \in \predictor(\scoresubset) \\ i\neq j}} \calibrationfunc_{ij}(\eps),
    \end{equation}
    where the set $\predictor(\scoresubset) \subseteq \outputdomain$ is defined as the set of labels that the predictor can predict for some feasible scores and $\calibrationfunc_{ij}$ is defined via minimization of the same objective as~\eqref{eq:calibrationfunc}, but w.r.t.\ a smaller domain:
    \begin{align}
    \label{eq:calibrationfunc:generalBreakingSymmetries}
    \calibrationfunc_{ij}(\eps)
    =
    \inf_{\scorev, \qv} \;& \excess\surrogateweighted(\scorev, \qv), \\
    \notag %
    \text{s.t.} \;& \lossweighted_i(\qv) \leq\lossweighted_j(\qv) - \eps,\\
    \notag %
    &\lossweighted_i(\qv) \leq \lossweighted_c(\qv),\;\; \forall c \in \predictor(\scoresubset), \\
    \notag %
    &\score_j \geq \score_c,\;\:\: \forall c \in \predictor(\scoresubset), \\
    \notag %
    & \scorev \in \scoresubset,\\
    \notag
    & \qv \in \simplex_\outputvarcard.
    \end{align}
    Here $\lossweighted_c(\qv) := (\lossmatrix \qv)_c$ is the expected loss if predicting label~$c$.
    Index~$i$ represents a label with the smallest expected loss while index~$j$ represents a label with the largest score.
\end{lemma}
\begin{proof}
    We use the notation $\scoresubset_j$ to define the set of score vectors~$\scorev$ where the predictor $\predictor(\scorev)$ takes a value~$j$, i.e., $\scoresubset_j := \{\scorev \in \scoresubset \mid \predictor(\scorev) = j\}$.
    The union of the sets $\scoresubset_j$, $j \in \predictor(\scoresubset)$, equals the whole set~$\scoresubset$.
    It is possible that sets $\scoresubset_j$ do not fully contain their boundary because of the usage of a particular tie-breaking strategy, but their closure can be expressed as $\overline{\scoresubset}_j := \{\scorev \in \scoresubset \mid \score_j \geq  \score_c, \forall c \in \predictor(\scoresubset)\}$.
    
    If $\scorev \in \scoresubset_j$, i.e. $j=\predictor(\scorev)$, then the feasible set of  probability vectors~$\qv$ for which a label~$i$ is one of the best possible predictions (i.e. $\excess\lossweighted(\scorev, \qv) = \ell_j(\qv) - \ell_i(\qv)  \geq \eps$) is $$\simplex_{\outputvarcard,i,j,\eps} := \{ \qv \in \simplex_{\outputvarcard} \mid \lossweighted_i(\qv) \leq \lossweighted_c(\qv), \forall c \in \predictor(\scoresubset); \ell_j(\qv) - \ell_i(\qv) \geq \eps\},$$ because $\inf_{\scorev' \in \scoresubset} \lossweighted(\scorev', \qv) = \min_{c \in \predictor(\scoresubset)} \lossweighted_c(\qv)$.
    
    The union of the sets~$\{\scoresubset_j \times \simplex_{\outputvarcard,i,j,\eps}\}_{i, j \in \predictor(\scoresubset)}$ thus exactly equals the feasibility set of the optimization problem~\eqref{eq:calibrationfunc}-\eqref{eq:calibrationfunc:epsConstr} (note that this is not true for the union of the sets~$\{\overline{\scoresubset}_j \times \simplex_{\outputvarcard,i,j,\eps}\}_{i, j \in \predictor(\scoresubset)}$, which can be strictly larger), thus we can rewrite the definition of the calibration function as follows:
    \begin{equation}
    \label{eq:calibrationfunc:proof}
    \calibrationfunc_{\surrogateloss,\lossmatrix,\scoresubset}(\eps) = \min_{\substack{i, j \in \predictor(\scoresubset) \\ i\neq j}} \inf_{\substack{\scorev \in \scoresubset_j,\\ \qv \in \simplex_{\outputvarcard,i,j,\eps}}} \excess\surrogateweighted(\scorev, \qv).
    \end{equation}
    To finish the proof, we use Lemma~27 of~\citep{zhang04} claiming that the function~$\excess\surrogateweighted(\scorev, \qv)$ is continuous w.r.t.\ both~$\qv$ and~$\scorev$, which allows us to substitute sets $\scoresubset_j$ in~\eqref{eq:calibrationfunc:proof} with their closures~$\overline{\scoresubset}_j$ without changing the value of the infimum.
\end{proof}

\section{Consistent surrogate losses}
\label{sec:specialLosses}
An ideal surrogate should not only be consistent, but also allow efficient optimization, by, e.g., being convex and allowing fast computation of stochastic gradients.
In this paper, we study a generalization to arbitrary multi-class losses of a surrogate loss class from~\citet[Section~4.4.2]{zhang04}\footnote{\citet{zhang04} refers to this surrogate as ``decoupled unconstrained background discriminative surrogate''. Note the~$\nicefrac{1}{\outputvarcard}$ scaling to make $\surrogateloss_{a,b}$ of order~$1$.}
that satisfies these requirements:
\begin{equation}
\label{eq:abLoss}
\surrogateloss_{a,b}( \scorev, \outputvarv  )
:=
\frac{1}{\outputvarcard}\sum\nolimits_{c = 1}^\outputvarcard \bigl( \lossmatrix(c, \outputvarv) a( \score_c ) + b( \score_c ) \bigr),
\end{equation}
where~$a, b : \R \to \R$ are convex functions.
A generic method to minimize this surrogate is to use any version of the SGD algorithm, while computing the stochastic gradient by sampling~$\outputvarv$ from the data generating distribution and a label~$c$ uniformly.
In the case of the quadratic surrogate~$\surrogatelossquad$, we proposed instead in the main paper to compute the sum over~$c$ analytically instead of sampling~$c$.

Extending the argument from~\citet{zhang04}, we show that the surrogates of the form~\eqref{eq:abLoss} are consistent w.r.t.\ a task loss~$\lossmatrix$ under some sufficient assumptions formalized in Theorem~\ref{th:abConsist}.
\begin{theorem}[Sufficient conditions for consistency]
    \label{th:abConsist}
    The surrogate loss~$\surrogateloss_{a,b}$ is consistent w.r.t.\ a task loss~$\lossmatrix$, i.e., $\calibrationfunc_{\surrogateloss_{a,b},\lossmatrix,\R^\outputvarcard}(\eps) > 0$ for any $\eps > 0$, under the following conditions on the functions~$a(\score)$ and $b(\score)$:
    \begin{enumerate}[noitemsep,topsep=-2mm]
        \item\label{item:thCond:smooth} The functions~$a$ and $b$ are convex and differentiable.
        \item\label{item:thCond:bounded} The function $ca(f) + b(f)$ is bounded from below and has a unique global minimizer (finite or infinite) for all $c \in [0, \Lmax]$.
        \item\label{item:thCond:second}  The functions~$a(f)$ and~$\frac{b'(f)}{a'(f)}$ are strictly increasing.
    \end{enumerate}
\end{theorem}
\begin{proof}
    Consider an arbitrary conditional probability vector~$\qv \in \simplex_k$.
    Assumption~\ref{item:thCond:bounded} then implies that the global minimizer~$\scorev^*$ of the conditional surrogate risk~$\surrogateweighted(\scorev, \qv)$ w.r.t.\ $\scorev$ is unique.
    Assumption~\ref{item:thCond:smooth} allows us to set the derivatives to zero and obtain~$\frac{b'(f_c^*)}{a'(f_c^*)} = - \lossweighted_c(\qv)$ where $\lossweighted_c(\qv) := (\lossmatrix \qv)_c$.
    Assumption~\ref{item:thCond:second} then implies that~$f_j^* \geq f_i^*$ holds if and only if $\lossweighted_j(\qv) \leq \lossweighted_i(\qv)$.
    
    Now, we will prove by contradiction that $\calibrationfunc(\eps) := \calibrationfunc_{\surrogateloss_{a,b},\lossmatrix,\R^\outputvarcard}(\eps) > 0$ for any $\eps > 0$.
    Assume that for some~$\eps > 0$ we have $\calibrationfunc(\eps) = 0$.
    Lemma~\ref{th:breakingSymmetries} then implies that for some $i, j \in \outputdomain$, $i \neq j$, we have $\calibrationfunc_{ij}(\eps) = 0$.
    Note that the domain of~\eqref{eq:calibrationfunc:generalBreakingSymmetries} defining $\calibrationfunc_{ij}$ is separable w.r.t.\ $\qv$ and $\scorev$.
    We can now rewrite~\eqref{eq:calibrationfunc:generalBreakingSymmetries} as
    \[
    \calibrationfunc_{ij}(\eps)
    =\!\!\!
    \inf_{\qv \in \simplex_{\outputvarcard,i,j,\eps}}\!\!\! \excess\surrogateweighted^*(\qv), \;
    \text{where}\; \excess\surrogateweighted^*(\qv) := \!\inf_{\scorev \in \overline{\scoresubset}_j}\! \excess\surrogateweighted(\scorev,\qv),
    \]
    where $\simplex_{\outputvarcard,i,j,\eps}$ and $\overline{\scoresubset}_j$ are defined in the proof of Lemma~\ref{th:breakingSymmetries}.
    Lemma~27 of~\citep{zhang04} implies that the function~$\excess\surrogateweighted^*(\qv)$ is a continuous function of~$\qv$.
    Given that~$\simplex_{\outputvarcard,i,j,\eps}$ is a compact set, the infimum is achieved at some point~$\qv^* \in \simplex_{\outputvarcard,i,j,\eps}$.
    For this~$\qv^*$, the global minimum w.r.t.\ $\scorev$ exists (Assumption~\ref{item:thCond:bounded}).
    The uniqueness of the global minimum implies that we have~$\score_{j}^* = \max_{c \in \outputdomain} \score_c^*$.
    The argument at the beginning of this proof then implies $\lossweighted_j(\qv^*) \leq \lossweighted_i(\qv^*)$ which contradicts the inequality $\lossweighted_i(\qv^*) \leq \lossweighted_j(\qv^*) - \eps$ in the definition of~$\simplex_{\outputvarcard,i,j,\eps}$.
\end{proof}

Note that Theorem~\ref{th:abConsist} actually proves that the surrogate~$\surrogateloss_{a,b}$ is order-preserving~\citep{zhang04}, which is a stronger property than consistency.

Below, we give several examples of possible functions $a(\score)$, $b(\score)$ that satisfy the conditions in Theorem~\ref{th:abConsist} and their corresponding~$\score^*(\lossweighted):=\argmin_{\scorev \in \R^\outputvarcard}  \surrogateweighted(\scorev, \qv)$ when $\lossweighted := \lossmatrix \qv$:
\begin{enumerate} %
    \item If $a(\score) = \score$, $b(\score) = \frac{\score^2}{2}$ then~$\score^*(\lossweighted) = -\lossweighted$, leading to our quadratic surrogate~\eqref{eq:quadrLoss}.
    \item If $a(\score) = \frac{1}{\Lmax}(\exp(\score) - \exp(-\score))$, $b(\score) = \exp(-\score)$ then~$\score^*(\lossweighted) = \frac12\log(1-\frac{1}{\Lmax}\lossweighted) - \frac12\log(\frac{1}{\Lmax}\lossweighted)$.
    \item If $a(\score) = \frac{1}{\Lmax}\score$, $b(\score) = \log(1 + \exp(-\score))$ then~$\score^*(\lossweighted) = \log(1-\frac{1}{\Lmax}\lossweighted) - \log(\frac{1}{\Lmax}\lossweighted)$.
\end{enumerate}

In the case of binary classification, these surrogates reduce to $\ltwonorm$-, exponential, and logistic losses, respectively.

\section{Bounds on the calibration function}
\label{sec:boundscalibrationFunction}

\subsection{Lower bound}
\label{sec:boundscalibrationFunction:lower}
\begin{reptheorem}{th:lowerBoundcalibrationFunction}[Lower bound on $\calibrationfunc_{\surrogatelossquad}$]
    \label{th:rep:lowerBoundcalibrationFunction}
     For any task loss~$\lossmatrix$, its quadratic surrogate~$\surrogatelossquad$, and a score subspace~$\scoresubset$ containing the column space of $\lossmatrix$, the calibration function can be lower bounded:
     \begin{equation*}
     \calibrationfunc_{\surrogatelossquad,\lossmatrix,\scoresubset}(\eps)
     \geq
     \frac{\eps^2}{2\outputvarcard \max_{i\neq j}\|\proj_{\scoresubset} \Delta_{ij}\|_2^2}
     \geq 
     \frac{\eps^2}{4\outputvarcard},
     \end{equation*}
     where $\proj_{\scoresubset}$ is the orthogonal projection on the subspace~$\scoresubset$ and $\Delta_{ij} = \unit_i - \unit_j \in \R^k$ with $\unit_c$ being the $c$-th basis vector of the standard basis in~$\R^\outputvarcard$.
\end{reptheorem}
\begin{proof}
    First, let us assume that the score subspace $\scoresubset$ is defined as the column space of a matrix~$\scorematrix \in \R^{\outputvarcard \times \scoresubspacedim}$, i.e., $\scorev(\scoreparamv)=\scorematrix\scoreparamv$.
    Lemma~\ref{th:quadSurrogateExcess} gives us expression~\eqref{eq:quadSurrogateExcessSimple} for~$\excess\surrogateweightedquad(\scorematrix\scoreparamv, \qv)$, which is jointly convex w.r.t.\ a conditional probability vector~$\qv$ and parameters~$\scoreparamv$.
    
    The optimization problem~\eqref{eq:calibrationfunc}-\eqref{eq:calibrationfunc:epsConstr} is non-convex because the constraint~\eqref{eq:calibrationfunc:epsConstr} on the excess risk depends of the predictor function~$\predictor(\scorev)$, see Eq.~\eqref{eq:predictor}, containing the~$\argmax$ operation.
    However, if we constrain the predictor to output label $j$, i.e., $\score_j \geq \score_c$, $\forall c$, and the label delivering the smallest possible expected loss to be $i$, i.e., $(\lossmatrix \qv)_i \leq (\lossmatrix \qv)_c$, $\forall c$, the problem becomes convex because all the constraints are linear and the objective is convex.
     Lemma~\ref{th:breakingSymmetries} in App.~\ref{sec:techLemmas} allows to bound the calibration function with the minimization w.r.t.\ selected labels $i$ and $j$, $\calibrationfunc_{\surrogatelossquad,\lossmatrix,\scoresubset}(\eps) \geq \min\limits_{i \neq j} \calibrationfunc_{ij}(\eps)$,\footnote{To simplify the statement of Theorem~\ref{th:lowerBoundcalibrationFunction}, we removed the constraints~$i,j \in \predictor(\scoresubset)$ from Lemma~\ref{th:breakingSymmetries} which said that we should consider only the labels that can be predicted with some feasible scores. A potentially tighter lower bound can be obtained by keeping the~$i,j \in \predictor(\scoresubset)$ constraint.} where~$\calibrationfunc_{ij}(\eps)$ is defined as follows:
    \begin{align}
    \label{eq:calibrationfunc:breakingSymmetries}
    \calibrationfunc_{ij}(\eps)
    =
    \min_{\scoreparamv, \qv} \;& \frac{1}{2 \outputvarcard}\|\scorematrix \scoreparamv + \lossmatrix \qv\|_2^2, \\
    \notag %
    \text{s.t.} \;& (\lossmatrix \qv)_i \leq (\lossmatrix \qv)_j - \eps,\\
    \notag %
    &(\lossmatrix \qv)_i \leq (\lossmatrix \qv)_c,\;\; \forall c \in \predictor(\scoresubset) \\
    \notag %
    &(\scorematrix \scoreparamv)_j \geq (\scorematrix \scoreparamv)_c,\; \forall c \in \predictor(\scoresubset) \\
    \notag %
    & \qv \in \simplex_\outputvarcard.
    \end{align}

    To obtain a lower bound, we relax~\eqref{eq:calibrationfunc:breakingSymmetries} by removing some of the constraints and arrive at
    \begin{align}
    \label{eq:calibrationfunc:breakingSymmetries:stripped}
    \calibrationfunc_{ij}(\eps)
    \geq
    \min_{\scoreparamv, \qv} \;& \frac{1}{2 \outputvarcard}\|\scorematrix \scoreparamv + \lossmatrix \qv\|_2^2, \\
    \label{eq:calibrationfunc:breakingSymmetries:stripped:epsConstr}
    \text{s.t.} \;& \Delta_{ij}^\transpose \lossmatrix \qv \leq - \eps,\\
    \label{eq:calibrationfunc:breakingSymmetries:fConstr}
    &\Delta_{ij}^\transpose \scorematrix \scoreparamv \leq 0,
    \end{align}
    where 
    $\Delta_{ij}^\transpose \lossmatrix \qv = (\lossmatrix \qv)_i - (\lossmatrix \qv)_j$,
    $\Delta_{ij}^\transpose \scorematrix \scoreparamv = (\scorematrix \scoreparamv)_i - (\scorematrix \scoreparamv)_j$, 
    and $\Delta_{ij} = \unit_i - \unit_j \in \R^k$ with $\unit_c \in \R^k$ being a vector of all zeros with $1$ at position~$c$.

    The constraint~\eqref{eq:calibrationfunc:breakingSymmetries:stripped:epsConstr} can be readily substituted with equality
    \begin{equation}
    \label{eq:calibrationfunc:breakingSymmetries:stripped:epsConstr:equality}
    \Delta_{ij}^\transpose \lossmatrix \qv = - \eps,
    \end{equation}
    without changing the minimum because multiplication of both~$\qv$ and~$\scoreparamv$ by the constant $\frac{-\eps}{\Delta_{ij}^\transpose \lossmatrix \qv} \in (0,1]$ preserves feasibility and can only decrease the objective~\eqref{eq:calibrationfunc:breakingSymmetries:stripped}.
    
    We now explicitly solve the resulting constraint optimization problem via the KKT optimality conditions.
    The stationarity constraints give us
    \begin{align}
    \label{eq:calibrationfunc:proof:KKT:dtheta}
    \frac{1}{\outputvarcard} \scorematrix^\transpose (\scorematrix \scoreparamv + \lossmatrix \qv) + \mu \scorematrix^\transpose \Delta_{ij} = 0,\\
    \label{eq:calibrationfunc:proof:KKT:dq}
    \frac{1}{\outputvarcard} \lossmatrix^\transpose (\scorematrix \scoreparamv + \lossmatrix \qv) + \nu \lossmatrix^\transpose \Delta_{ij} = 0;
    \end{align}
    the complementary slackness gives $\mu \Delta_{ij}^\transpose \scorematrix \scoreparamv = 0$ and the feasibility constraints give~\eqref{eq:calibrationfunc:breakingSymmetries:stripped:epsConstr:equality}, \eqref{eq:calibrationfunc:breakingSymmetries:fConstr}, and $\mu \geq 0$.
    
    Equation~\eqref{eq:calibrationfunc:proof:KKT:dtheta} allows to compute
    \begin{equation}
    \label{eq:calibrationfunc:proof:theta}
    \scoreparamv = -(\scorematrix^\transpose \scorematrix)^\pinv(\outputvarcard \mu \scorematrix^\transpose \Delta_{ij} + \scorematrix^\transpose \lossmatrix \qv).
    \end{equation}
    By substituting \eqref{eq:calibrationfunc:proof:theta} into~\eqref{eq:calibrationfunc:proof:KKT:dq} and by using the identity (because $\lossmatrix \in \colspace(\scorematrix)$):
    \begin{equation}
    \label{eq:ProjFL}
    \proj_{\scorematrix} \lossmatrix = \scorematrix (\scorematrix^\transpose \scorematrix)^\pinv \scorematrix^\transpose \lossmatrix = \lossmatrix,
    \end{equation}
    we get
    $
    (\mu -\nu) \lossmatrix^\transpose \Delta_{ij} = 0.
    $
    If $\lossmatrix^\transpose \Delta_{ij} = 0$, the problem~\eqref{eq:calibrationfunc:breakingSymmetries:stripped}, \eqref{eq:calibrationfunc:breakingSymmetries:fConstr}, \eqref{eq:calibrationfunc:breakingSymmetries:stripped:epsConstr:equality} is infeasible for $\eps > 0$ implying $\calibrationfunc_{ij}(\eps) = +\infty$. Otherwise, we have $\mu = \nu$.
    
    From~\eqref{eq:calibrationfunc:proof:theta} and~\eqref{eq:ProjFL}, we also have that:
    \begin{equation}
    \label{eq:LowerBoundObjectiveVector}
    \scorematrix \scoreparamv + \lossmatrix \qv = -\outputvarcard \mu \proj_\scorematrix \Delta_{ij} .
    \end{equation}
    
    By plugging~\eqref{eq:calibrationfunc:proof:theta} into the complementary slackness condition and combining with~\eqref{eq:calibrationfunc:breakingSymmetries:stripped:epsConstr:equality}, we get
    \[
    \mu^2 \outputvarcard \| \proj_\scorematrix \Delta_{ij} \|_2^2 = \mu \eps
    \]
    implying that either $\mu = 0$ or $\mu \outputvarcard \| \proj_\scorematrix \Delta_{ij} \|_2^2 = \eps$.
    In the first case, Eq.~\eqref{eq:LowerBoundObjectiveVector} implies $\scorematrix \scoreparamv = -\lossmatrix \qv$ making satisfying both \eqref{eq:calibrationfunc:breakingSymmetries:stripped:epsConstr:equality} and~\eqref{eq:calibrationfunc:breakingSymmetries:fConstr} impossible. Thus, the later is satisfied implying that the objective~\eqref{eq:calibrationfunc:breakingSymmetries:stripped} is equal to\footnote{The possibility $\proj_\scorematrix \Delta_{ij}=0$ is also covered by this equation with the convention that $1/0=\infty$ (in this case, $\mu^*=\infty$).}
    \[
    \frac{1}{2 \outputvarcard}\|\scorematrix \scoreparamv + \lossmatrix \qv\|_2^2
    =
    \frac{\eps^2}{2 \outputvarcard \| \proj_\scorematrix \Delta_{ij} \|_2^2}.
    \]
    Finally, orthogonal projections contract the $\ltwonorm$-norm, thus $\|\proj_\scorematrix \Delta_{ij}\|_2^2 \leq 2$, which gives the second lower bound in the statement of the theorem and finishes the proof.
\end{proof}
 
\subsection{Upper bound}
\label{sec:boundscalibrationFunction:upper}
\begin{reptheorem}{th:upperBoundcalibrationFunction}[Upper bound on $\calibrationfunc_{\surrogatelossquad}$]
    \label{th:rep:upperBoundcalibrationFunction}
    If a loss matrix~$\lossmatrix$ with $\Lmax > 0$ defines a pseudometric\footnoteref{footnote:pseudometric} on labels and there are no constraints on the scores, i.e., $\scoresubset = \R^{\outputvarcard}$, then the calibration function for the quadratic surrogate~$\surrogatelossquad$ can be upper bounded:
    \begin{equation*}
    \calibrationfunc_{\surrogatelossquad,\lossmatrix,\scoresubset}(\eps)
    \leq  \frac{\eps^2}{2\outputvarcard}, \quad 0 \leq \eps \leq \Lmax.
    \end{equation*}
\end{reptheorem}
\begin{proof}
    After applying Lemmas~\ref{th:quadSurrogateExcess} and~\ref{th:breakingSymmetries}, we arrive at
    \begin{align}
    \label{eq:upperBoundcalibrationFunction:breakingSymmetries}
    \calibrationfunc_{ij}(\eps)
    =
    \inf_{\scorev, \qv} \;& \frac{1}{2\outputvarcard} \| \scorev + \lossmatrix\qv \|^2, \\
    \notag %
    \text{s.t.} \;& \lossweighted_i(\qv) \leq\lossweighted_j(\qv) - \eps,\\
    \notag %
    &\lossweighted_i(\qv) \leq \lossweighted_c(\qv),\;\; \forall c \in \outputdomain, \\
    \notag %
    &\score_j \geq \score_c,\;\:\: \forall c \in \outputdomain, \\
    \notag %
    & \scorev \in \R^\outputvarcard, \quad \qv \in \simplex_\outputvarcard.
    \end{align}
    We now consider labels $i$ and $j$ such that $\lossmatrix_{ij}=\Lmax > 0$ and the point $\q_i = \frac12 + \frac{\eps}{2\lossmatrix_{ij}}$, $\q_j = \frac12 - \frac{\eps}{2\lossmatrix_{ij}}$ (non-negative for $\eps \leq \Lmax$). We let $\q_c = 0$ for $c \not\in \{i,j\}$, $\score_j = \score_i = -\lossweighted_i(\qv)$ and $\score_c = -\lossweighted_c(\qv)$ for $c \not\in \{i,j\}$. We now show that this assignment is feasible. 
    
    We have $\lossweighted_j(\qv) = \q_i \lossmatrix_{ji} + \q_j \lossmatrix_{jj}=\q_i \lossmatrix_{ji} = \q_i \lossmatrix_{ij}$ by symmetry of $\lossmatrix$. Similarly, $\lossweighted_i(\qv) = \q_j \lossmatrix_{ij}$ and thus
    \[
    \lossweighted_j(\qv) - \lossweighted_i(\qv) = \lossmatrix_{ij} \frac{\eps}{\lossmatrix_{ij}} = \eps.
    \]
    We also have 
    \begin{equation*}
    \lossweighted_c(\qv) - \lossweighted_i(\qv) = \q_i \lossmatrix_{ci} + \q_j \lossmatrix_{cj} - \q_j \lossmatrix_{ij} \geq  \q_j (\lossmatrix_{ic} + \lossmatrix_{cj} - \lossmatrix_{ij}) \geq 0.
    \end{equation*}
    The first inequality uses  $\q_i \geq \q_j$ and the second inequality uses the fact that $\lossmatrix$ satisfies the triangle inequality (as a pseudometric).
    Finally, $\score_j - \score_c = - \lossweighted_i(\qv) + \lossweighted_c(\qv) \geq 0$.
    
    We thus have shown that the defined point is feasible, so we compute its objective value.
    We have
    \[
    \frac{1}{2\outputvarcard} \| \scorev + \lossmatrix\qv \|^2 = \frac{1}{2\outputvarcard}(\lossweighted_j(\qv) - \lossweighted_i(\qv))^2 = \frac{\eps^2}{2\outputvarcard},
    \]
    which completes the proof.
\end{proof}

\subsection{Computation of the lower bounds for specific task losses}
\label{sec:bounds:interpretations}
\textbf{0-1 loss.}
Let $\lossmatrix_{01}$ denote the loss matrix of the 0-1 loss, i.e., $\lossmatrix_{01}(i, j) := [i \neq j]$.\footnoteref{footnote:Iverson}
It is convenient to rewrite it with a matrix notation $\lossmatrix_{01} = \one_\outputvarcard \one_\outputvarcard^\transpose - \id_\outputvarcard$, where $\one_\outputvarcard \in \R^\outputvarcard$ is the vector of all ones and $\id_\outputvarcard \in \R^{\outputvarcard\times\outputvarcard}$ is the identity matrix.
We have $\rank(\lossmatrix_{01}) = \outputvarcard$ (for $\outputvarcard \geq 2$), thus $\colspace(\lossmatrix) = \R^\outputvarcard$.
By putting no constraints on the scores, we can easily apply Theorem~\ref{th:lowerBoundcalibrationFunction} and obtain the lower bound of~$\frac{\eps^2}{4\outputvarcard}$, which is shown to be tight in Proposition~\ref{th:calibrationFunction:01loss} of Section~\ref{sec:exactTranferProofs:01loss}.

\textbf{Block 0-1 loss.}
We use the symbol $\lossmatrix_{01,\numblocks}$ to denote the loss matrix of the block 0-1 loss with $\numblocks$~blocks, i.e., $\lossmatrix_{01,\numblocks}(i, j) := [\text{$i$ and $j$ are not in the same block}]$.
We use $\blocksize_v$ to denote the size of block~$v$, $v=1,...,\numblocks$, and then $\blocksize_1 +\dots + \blocksize_\numblocks = \outputvarcard$.
In the case when all the blocks are of equal sizes, we denote their size by~$\blocksize$ and have~$\outputvarcard = \numblocks \blocksize$.

With a matrix notation, we have $\lossmatrix_{01,\numblocks} = \one_\outputvarcard \one_\outputvarcard^\transpose - UU^\transpose$ where the columns of the matrix $U\in \R^{\outputvarcard \times \numblocks}$ are indicators of the blocks.
We have $\rank(\lossmatrix_{01,\numblocks}) = \numblocks$ and can simply define $\scoresubset_{01,\numblocks} := \colspace(\scorematrix_{01,\numblocks})$ with  $\scorematrix_{01,\numblocks} := U$.
If we assume that all the blocks have equal size, then we have $U^\transpose U = \blocksize \id_\numblocks$ and $\|\proj_{\scoresubset_{01,\numblocks}}\Delta_{ij}\|_2^2=\frac{2}{\blocksize}$ if labels~$i$ and~$j$ belong to different blocks, while $\proj_{\scoresubset_{01,\numblocks}}\Delta_{ij}=0$ if $i$ and~$j$ belong to the same block.  This leads to the lower bound~$\frac{\eps^2}{4\numblocks}$, which is shown to be tight in Proposition~\ref{th:calibrationFunction:block01loss:hardConstr} of Section~\ref{sec:exactTranferProofs:block01loss}.

\textbf{Hamming loss.}
Consider the (normalized) Hamming loss between tuples of $\hamminglen$~binary variables, where~$\hat{\outputvar}_\hammingindex$ and~$\outputvar_\hammingindex$ are the $\hammingindex$-th variables of a prediction~$\hat{\outputvarv}$ and a correct label~$\outputvarv$, respectively:
\begin{align}
\label{eq:hammingLoss}
\lossmatrix_{\hamming,\hamminglen}(\hat{\outputvarv}, \outputvarv)
&:=
\frac{1}{\hamminglen}\sum\nolimits_{\hammingindex=1}^\hamminglen [\hat{\outputvar}_\hammingindex \neq \outputvar_\hammingindex]
\\
\notag &=
\frac{1}{\hamminglen}\sum\nolimits_{\hammingindex=1}^\hamminglen ([\hat{\outputvar}_\hammingindex = 0] [\outputvar_\hammingindex = 1] + [\hat{\outputvar}_\hammingindex = 1] [\outputvar_\hammingindex = 0])
\\
\notag &=
\frac{1}{\hamminglen}\sum\nolimits_{\hammingindex=1}^\hamminglen (1-[\hat{\outputvar}_\hammingindex = 1]) [\outputvar_\hammingindex = 1] + [\hat{\outputvar}_\hammingindex = 1] [\outputvar_\hammingindex = 0])
\\
\notag &=
\frac{1}{\hamminglen}\sum\nolimits_{\hammingindex=1}^\hamminglen  [\outputvar_\hammingindex = 1] +  
\frac{1}{\hamminglen}\sum\nolimits_{\hammingindex=1}^\hamminglen   \left([\outputvar_\hammingindex = 0]-[\outputvar_\hammingindex = 1]\right) \,\, [\hat{\outputvar}_\hammingindex = 1]
\\
\notag &=
\alpha_0(\outputvarv) + \sum\nolimits_{\hammingindex=1}^\hamminglen \alpha_\hammingindex(\outputvarv) [\hat{\outputvar}_\hammingindex = 1],
\end{align}
The vectors $\alphav_\hammingindex(\cdot)$ depend only on the column index of the loss matrix.
The decomposition~\eqref{eq:hammingLoss} implies that 
$
\scoresubset_{\hamming,\hamminglen}
:=
\colspace(\scorematrix_{\hamming,\hamminglen})
$
equals to
$
\colspace(\lossmatrix_{\hamming,\hamminglen})
$
for
$
\scorematrix_{\hamming,\hamminglen} := [ \frac12 \one_{2^\hamminglen}, \hv^{(1)}, \dots, \hv^{(\hamminglen)} ],
$
$
(\hv^{(\hammingindex)})_{\hat{\outputvarv}} := [\hat{\outputvar}_\hammingindex = 1],
$
$\hammingindex = 1,\dots,\hamminglen$.
We also have that $\rank(\lossmatrix_{\hamming,\hamminglen}) = \rank(\scorematrix_{\hamming,\hamminglen}) = \hamminglen + 1$.

In Section~\ref{sec:hammingLossProps}, we show that $\max_{i\neq j}\|\proj_{\scoresubset_{\hamming,\hamminglen}} \Delta_{ij}\|_2^2 = \frac{4\hamminglen}{2^\hamminglen}$.
By plugging this identity into the lower bound~$\eqref{eq:lowerBoundcalibrationFunction}$, we get $\calibrationfunc_{\surrogatelossquad,\lossmatrix_{\hamming,\hamminglen},\scoresubset_{\hamming,\hamminglen}} \geq \frac{\epsilon^2}{8\hamminglen}$, which appears to be tight according to  Proposition~\ref{th:calibrationFunction:hammingLoss:hardConstr} of Section~\ref{sec:exactTranferProofs:hamming}.

\textbf{Non-tight cases.}
In the cases of the block 0-1 loss and the mixed 0-1 and block 0-1 loss   (Propositions~\ref{th:calibrationFunction:block01loss} and~\ref{th:calibrationFunction:mixedLoss}, respectively), we observe gaps between the lower bound~$\eqref{eq:lowerBoundcalibrationFunction}$ and the exact calibration functions, which shows the limitations of the bound.
In particular, it cannot detect level-$\consbreakpoint$ consistency for $\consbreakpoint > 0$ (see Definition~\ref{def:consistency}) and does not change when the loss changes, but the score subspace stays the same.

\section{Exact calibration functions for quadratic surrogate}
\label{sec:exactTranferProofs}
This section presents our derivations for the exact values of the calibration functions for different losses.
While doing these derivations, we have used numerical simulations and symbolic derivations to check for correctness.
Our numerical and symbolic tools are available online.\footnote{\url{https://github.com/aosokin/consistentSurrogates_derivations}}

\subsection{0-1 loss}
\label{sec:exactTranferProofs:01loss}
\begin{proposition}
    \label{th:calibrationFunction:01loss}
    Let $\lossmatrix_{01}$ be the 0-1 loss, i.e., $\lossmatrix_{01}(i, j) = [i \neq j]$.
    Then, the calibration function equals the following quadratic function w.r.t.\ $\eps$:
    \begin{equation}
    \label{eq:calibrationFunction:01loss}
    \calibrationfunc_{\surrogatelossquad,\lossmatrix_{01},\R^k}(\eps) = \frac{\eps^2}{4\outputvarcard},
    \quad 0 \leq \eps \leq 1.
    \end{equation}
\end{proposition}
Note that in the case of binary classification, the function~\eqref{eq:calibrationFunction:01loss} is equal to the calibration function for the least squares and truncated least squares surrogates~\citep{bartlett06convexity,steinwart07}.

\begin{proof}
First, Lemma~\ref{th:quadSurrogateExcess} with $\scoresubset = \R^{\outputvarcard}$ and $\scorematrix = \id_\outputvarcard$ gives us the expression
\begin{equation}
\label{eq:quadSurrogateExcessSimple:noConstrainst}
\excess\surrogateweightedquad(\scorematrix\scoreparamv, \qv) = \frac{1}{2\outputvarcard}\|\scorev + \lossmatrix \qv \|_2^2,
\end{equation}
with $\scorev = \scoreparamv \in \R^{\outputvarcard}$.

We now reduce the optimization problem~\eqref{eq:calibrationfunc}-\eqref{eq:calibrationfunc:epsConstr} to a convex one by using Lemma~\ref{th:breakingSymmetries} and by writing 
$
\calibrationfunc_{\surrogatelossquad,\lossmatrix_{01},\R^k}(\eps) = \min_{i \neq j \in \outputdomain}\calibrationfunc_{ij}(\eps)
$, which holds because $\predictor(\R^k) = \outputdomain$.
Because of the symmetries of the 0-1 loss, all the choices of~$i$ and~$j$ give the same (up to a permutation of labels) optimization problem to compute~$\calibrationfunc_{ij}(\eps)$.
The definition of the 0-1 loss implies $(\lossmatrix \qv)_c = 1-q_c$, which simplifies the excess of the expected task loss appearing in~\eqref{eq:calibrationfunc:epsConstr} to
$
\excess\lossweighted(\scorev, \qv) = (\lossmatrix \qv)_j - (\lossmatrix \qv)_i = \q_i - \q_j.
$
After putting all these together, we get
\begin{align}
\label{eq:calibrationfunc:01loss:breakingSymmetries}
\calibrationfunc_{ij}(\eps)
=
\min_{\scorev, \qv} \;& \frac{1}{2 \outputvarcard} \sum_{c=1}^{\outputvarcard} (\score_c + 1 - \q_c)^2, \\
\notag %
\text{s.t.} \;& \q_i \geq \q_j + \eps,\\
\notag %
&\q_i \geq \q_c,\;\; c = 1,\dots,\outputvarcard, \\
\notag %
&\score_j \geq \score_c,\; c = 1,\dots,\outputvarcard, \\
\notag %
& \sum_{c=1}^{\outputvarcard} \q_c = 1, \quad \q_c \geq 0.
\end{align}
We claim that there exists an optimal point of~\eqref{eq:calibrationfunc:01loss:breakingSymmetries}, $\scorev^*$, $\qv^*$, such that $\q^*_c=0$, $c\not\in\{i, j\}$, $\q^*_i = \frac12 + \frac{\eps}{2}$,  $\q^*_j = \frac12 - \frac{\eps}{2}$; $\score^*_c = -1$, $c\not\in\{i, j\}$, $\score^*:= \score^*_i = \score^*_j$.
Note that apart from the specific value of $\score^*$, this is the same point used to prove the upper bound of Theorem~\ref{th:upperBoundcalibrationFunction}.
After proving this, we will minimize the objective w.r.t.\ remaining scores at this point.\footnote{Note that just showing the feasibility of the assigned values~$\qv^*$ and~$\scorev^*$ give us an upper bound on the calibration function.
In the case of the 0-1 loss, it appears that this upper bound matches the lower bound provided by Theorem~\ref{th:lowerBoundcalibrationFunction}, so we do not need to prove optimality explicitly.
However, we still give this proof as a simple illustration of the proof technique as its structure will be re-used also for the cases when the bound of Theorem~\ref{th:lowerBoundcalibrationFunction} is not tight.}

First, if any $\q^*_c=\delta > 0$, $c\not\in\{i, j\}$, we can safely move this probability mass to $\q_i$ and $\q_j$ with the operation
\begin{align*}
&\q^*_c := \q^*_c -\delta = 0, \; &&\q^*_i := \q^*_i +\frac{\delta}{2}, \; &&\q^*_j := \q^*_j +\frac{\delta}{2},\\
&\score^*_c := \score^*_c -\delta, \; &&\score^*_i := \score^*_i +\frac{\delta}{2}, \; &&\score^*_j := \score^*_j +\frac{\delta}{2},
\end{align*}
which keeps all the constraints of~\eqref{eq:calibrationfunc:01loss:breakingSymmetries} feasible and does not change the objective value.

Second, all the scores~$\score^*_c$ have to belong to the segment~$[-1, 0]$ otherwise clipping them will decrease the objective.
With this, setting $\score^*_c := -1$, $c\not\in\{i, j\}$ can only decrease the objective and will not violate the constraints.

We now show that the equality $\q^*_i = \q^*_j + \eps$ can hold at the optimum.
Indeed, if $\q^*_i - \q^*_j = \delta' >  \eps$, the operation
\begin{align}
\label{eq:tfproofs:truncatingToEps}
\q_i^* &:= \q_i^* - \frac{\delta'-\eps}{2}, &
\q_{j}^* &:= \q_{j}^* + \frac{\delta'-\eps}{2}, \\
\notag\score_i^* &:= \score_i^* - \frac{\delta'-\eps}{2}, &
\score_j^* &:= \score_j^* + \frac{\delta'-\eps}{2}.
\end{align}
keeps the objective the same and maintains the feasibility constraints. So combining with $\q^*_i + \q^*_j = 1$, we can now conclude that $\q^*_i = \frac12 + \frac{\eps}{2}$,  $\q^*_j = \frac12 - \frac{\eps}{2}$ is an optimal point.

We now show that the equality $\score^*_i = \score^*_j$ can hold at the optimum. First, we know that the values $\score^*_i$ and $\score^*_j$ belong to the segment $[\q^*_j-1,\q^*_i-1]$, otherwise we can always truncate the values to the borders of the segment and get an improvement of the objective. Finally, since the inequality $\score^*_j \geq \score^*_i$ must hold, we conclude that $\score^*_i = \score^*_j := \score^*$ so that $\score^*_i$ is closest to its target $\q^*_i-1$ to minimize the objective. 

At the optimal point defined above, it remains to find the value~$\score^*$ delivering the minimum of the objective.
We can achieve this by computing 
\[
\calibrationfunc_{ij}(\eps) = \frac{1}{2\outputvarcard} \min_{\score \in [-1, 0]}  (f +\frac12 -\frac{\eps}{2})^2 + (f +\frac12 +\frac{\eps}{2})^2,
\]
which implies $f^* = -0.5$ and $\calibrationfunc_{\surrogatelossquad,\lossmatrix_{01},\R^k}(\eps) = \frac{\eps^2}{4\outputvarcard}$.
\end{proof}

\textbf{Remark.} We note that the conditional distribution used in the proof above, $\q_i = \frac12 + \frac{\eps}{2}$,  $\q_j = \frac12 - \frac{\eps}{2}$, $\q_c = 0$, $c\not\in\{i, j\}$, is somewhat unsatisfying from the perspective of explaining why learning the 0-1 loss might be difficult. Indeed, it looks like a gradient based learning algorithm that would start with all values $\score_c = -1$ would at the end only optimize over $\score_i$ and $\score_j$ as the gradient with respect to $\score_c$ for $c \notin \{i,j\}$ would stay at zero in~$\surrogatelossquad( \scorev, \outputvarv )$~\eqref{eq:quadrLoss} given that only $i$ or $j$ could appear in $\outputvarv$. 
From this observation, one could think that the calibration function perspective is misleading as SGD could have faster convergence rate than predicted by the worst case for this situation.
Fortunately, one can easily check that the point $\q_i = \frac13 + \frac{\eps}{2}$,  $\q_j = \frac13 - \frac{\eps}{2}$, $\q_c = \frac{1}{3(k-2)}$ for  $c\not\in\{i, j\}$, $\score_i = \score_j = \frac{1}{3}$ and $\score_c = -\lossweighted_c(\qv)$ for $c\not\in\{i, j\}$ is feasible for~\eqref{eq:calibrationfunc:01loss:breakingSymmetries} and yields the same optimal value of $\frac{\eps^2}{4\outputvarcard}$ for the objective, thus providing another example where the exponential multiclass nature is more readily apparent and cannot be fixed by some ``natural initialization'' of the learning algorithm.

\subsection{Block 0-1 loss}
\label{sec:exactTranferProofs:block01loss}
Recall that $\lossmatrix_{01,\numblocks}$ is the block 0-1 loss, i.e., $\lossmatrix_{01,\numblocks}(i, j) = [\text{$i$ and $j$ are not in the same block}]$.
We use $\numblocks$ to denote the total number of blocks and $\blocksize_v$ to denote the size of block~$v$, $v=1,...,\numblocks$.
In this section, we compute the calibration functions for the case of unconstrained scores (Proposition~\ref{th:calibrationFunction:block01loss}) and for the case of the scores belonging to the column span of the loss matrix (Proposition~\ref{th:calibrationFunction:block01loss:hardConstr}).
\begin{proposition}
    \label{th:calibrationFunction:block01loss}
    Without constraints on the scores, the calibration function for the block 0-1 loss equals the following quadratic function w.r.t.\ $\eps$:
    \[
    \calibrationfunc_{\surrogatelossquad,\lossmatrix_{01,\numblocks},\R^k}(\eps) = \frac{\eps^2}{4\outputvarcard} \min_{v=1,...,\numblocks}\frac{2 \blocksize_v}{\blocksize_v + 1} \leq \frac{\eps^2}{2\outputvarcard},
    \quad 0 \leq \eps \leq 1.
    \]
\end{proposition}
Note that when $\blocksize_v = 1$ for some $v$, we have  $\calibrationfunc_{\surrogatelossquad,\lossmatrix_{01,\numblocks},\R^k}(\eps)$ matching to the $\frac{\eps^2}{4\outputvarcard}$ lower bound of Theorem~\ref{th:lowerBoundcalibrationFunction}.
When $\blocksize_v \to \infty$ for all blocks, we have $\calibrationfunc_{\surrogatelossquad,\lossmatrix_{01,\numblocks},\R^k}(\eps)$ matching to the $\frac{\eps^2}{2\outputvarcard}$ upper bound of Theorem~\ref{th:upperBoundcalibrationFunction}.

\begin{proof}
    This proof is of the same structure as the proof of Proposition~\ref{th:calibrationFunction:01loss} above.
    
    We use $\blockOfLabel{i} \in 1,\dots,\numblocks$ to denote the block to which label~$i$ belongs and $\labelsOfBlock{v}$ to denote the set of labels that belong to block~$v$.
    We also use $\Q_v$, $v\in 1,\dots, \numblocks$, as a shortcut to $\sum_{i \in \labelsOfBlock{v}} \q_i$, which is the total probability mass on block~$v$.
    
    We start by noting that the $i$-th component of the vector $(\lossmatrix_{01,\numblocks}) \qv$ equals $1-\Q_{\blockOfLabel{i}}$.
    By applying Lemmas~\ref{th:quadSurrogateExcess},~\ref{th:breakingSymmetries}, we get
    \begin{align}
    \label{eq:calibrationfunc:block01loss:breakingSymmetries}
    \calibrationfunc_{ij} (\eps) = \min_{\scorev, \qv} \quad & \frac{1}{2\outputvarcard} \sum_{v=1}^\numblocks \sum_{c \in \labelsOfBlock{v}} (\score_c + 1 - \Q_{\blockOfLabel{c}})^2, \\
    \label{eq:calibrationfunc:block01loss:breakingSymmetries:epsConstr}
    & \Q_{\blockOfLabel{i}} - \Q_{\blockOfLabel{j}} \geq \eps, \\
    \notag & \Q_{\blockOfLabel{i}} \geq \Q_u, \;u=1,\dots,\numblocks, \\
    \notag & \score_j \geq f_c, \;c=1,\dots,\outputvarcard, \\
    \notag & \sum\nolimits_{c=1}^{\outputvarcard} \q_c = 1, \quad \q_c \geq 0.
    \end{align}
    Analogously to Proposition~\ref{th:calibrationFunction:01loss}, we claim that there exists an optimal point of~\eqref{eq:calibrationfunc:block01loss:breakingSymmetries} such that $\q_c=0$, $c\not\in\{i,j\}$; $q_i = 0.5 + \frac{\eps}{2} = \Q_{\blockOfLabel{i}}$;  $\q_{j} = 0.5 - \frac{\eps}{2} = \Q_{\blockOfLabel{j}}$; $\score_c = -1$, $c \not\in \labelsOfBlock{ij} := \labelsOfBlock{\blockOfLabel{i}} \cup \labelsOfBlock{\blockOfLabel{j}}$.

    At first, note that if $\blockOfLabel{i} = \blockOfLabel{j}$, then the constraint~\eqref{eq:calibrationfunc:block01loss:breakingSymmetries:epsConstr} is never feasible, so we'll assume that $\blockOfLabel{i} \neq \blockOfLabel{j}$.
    
    We will now show that we can consider only configurations with all the probability mass on the two selected blocks.
    Consider some optimal point $\scorev^*$, $\qv^*$ and denote with $\delta = \sum_{c\in \outputdomain \setminus \labelsOfBlock{ij}} \q_c^*$ the probability mass on the unselected blocks.
    The operation
    \begin{align*}
    \score_c^* &:= \score_c^* + \frac{\delta}{2}, \; c \in \labelsOfBlock{ij},& 
    \score_c^* &:= -1, \; c\not\in \labelsOfBlock{ij} \\
    \q_i^* &:= \q_i^* + \frac{\delta}{2}, 
    \q_{j}^* := \q_{j}^* + \frac{\delta}{2}, &
    \q_c^* &:= 0, \; c\not\in \labelsOfBlock{ij} 
    \end{align*}
    can only decrease the objective of~\eqref{eq:calibrationfunc:block01loss:breakingSymmetries} because the summands corresponding to the unselected blocks are set to zero.
    All the constraints stay feasible and the summands corresponding to the selected blocks keep their values.
    
    The probability mass within the block~$\blockOfLabel{i}$ can be safely moved to $\q_i^*$ without changing the objective or violating any constraints.
    Analogously, the probability mass within the block~$\blockOfLabel{j}$ can be safely moved to $\q_j^*$.
    By reusing the operation~\eqref{eq:tfproofs:truncatingToEps}, we can now ensure that $\q^*_i = \q^*_j + \eps$ and thus that $\q^*_i = \frac{1}{2} + \frac{\eps}{2}$ and $\q^*_{j} = \frac{1}{2} - \frac{\eps}{2}$.

    At the point defined above, we now minimize the objective~\eqref{eq:calibrationfunc:block01loss:breakingSymmetries} w.r.t.\ $\score_c$, $c \in \labelsOfBlock{ij}$.
    At an optimal point, all values~$\score_c^*$, $c \in \labelsOfBlock{ij}$, belong to the segment $[\Q_{\blockOfLabel{j}}^* - 1, \Q_{\blockOfLabel{i}}^* - 1]$, otherwise we can always truncate the values to the borders of the segment and get an improvement of the objective.
    For all the scores $\score_c^*$, $c \neq j$, the following identity holds
    \begin{equation}
    \label{eq:calibrationfunc:block01loss:scoreInSelBlocks}
    \score_c^* = \begin{cases}
    \Q_{\blockOfLabel{c}}^* - 1, \; \text{if $\Q_{\blockOfLabel{c}}^* - 1 < \score_{j}^*$}, \\
    \score_{j}^*.
    \end{cases}
    \end{equation}
    Combining with the segment constraint, it implies that in the block of the label~$i$, we have $\score_c^* = \score_{j}^*$, $c \in \labelsOfBlock{\blockOfLabel{i}}$, and, in the block of the label~$j$, we have $\score_c^* = \Q_{\blockOfLabel{j}}^* - 1$, $c \in \labelsOfBlock{\blockOfLabel{j}} \setminus j$.
    
    By plugging the obtained values of $\q_c^*$ and $\score_c^*$ into~\eqref{eq:calibrationfunc:block01loss:breakingSymmetries} and denoting the value $\score_{j}^* + 0.5$ with $\tilde{\score}$, we get
    \begin{align}
    \label{eq:block01loss:proofSquare:obj2}
    \calibrationfunc_{ij}(\eps) = \min_{\tilde{\score}} \; &\frac{1}{2\outputvarcard}  \left( \blocksize_{\blockOfLabel{i}} (\tilde{\score} - \frac{\eps}{2})^2 + (\tilde{\score} + \frac{\eps}{2})^2 \right),\\
    \notag
    \mbox{s.t.}\; & \tilde{\score} \in [-\frac{\eps}{2}, \frac{\eps}{2}].
    \end{align}
    By setting the derivative of the objective~\eqref{eq:block01loss:proofSquare:obj2} to zero, we get
    $$
    \tilde{\score} = \frac{\eps}{2} \frac{\blocksize_{\blockOfLabel{i}} - 1}{\blocksize_{\blockOfLabel{i}} + 1},
    $$
    which belongs to the segment $[-\frac{\eps}{2}, \frac{\eps}{2}]$.
    We compute the function value at this point:
    $$
    \calibrationfunc_{ij} (\eps) = \frac{\eps^2}{4\outputvarcard} \frac{2\blocksize_{\blockOfLabel{i}}}{\blocksize_{\blockOfLabel{i}}+1},
    $$
    which finishes the proof.
\end{proof}

\begin{proposition}
    \label{th:calibrationFunction:block01loss:hardConstr}
    Let the scores~$\scorev$ be piecewise constant on the blocks of the loss, i.e.\ belong to the subspace~$\scoresubset_{01,\numblocks} = \colspace(\lossmatrix_{01,\numblocks}) \subseteq \R^k$.
    Then, the calibration function equals the following quadratic function w.r.t.\ $\eps$:
    \[
    \calibrationfunc_{\surrogatelossquad,\lossmatrix_{01,\numblocks},\scoresubset_{01,\numblocks}}(\eps)
    =
    \frac{\eps^2}{4\outputvarcard} \min_{v \neq u} \frac{2 \blocksize_v \blocksize_u}{\blocksize_v + \blocksize_u},
    \quad 0 \leq \eps \leq 1.
    \]
    If all the blocks are of the same size, we have
    $
    \calibrationfunc_{\surrogatelossquad,\lossmatrix_{01,\numblocks},\scoresubset_{01,\numblocks}}(\eps)
    =
    \frac{\eps^2}{4\numblocks}
    $
    where $\numblocks$ is the number of blocks.
\end{proposition}
\begin{proof}
    The constraints on scores $\scorev \in \scoresubset_{01,\numblocks}$ simply imply that the scores within all the blocks are equal.
    Having this in mind, the proof exactly matches the proof of Proposition~\ref{th:calibrationFunction:block01loss} until the argument around Eq.~\eqref{eq:calibrationfunc:block01loss:scoreInSelBlocks}.
    Now we cannot set the scores of the block~$\blockOfLabel{j}$ to different values, and, thus they are all equal to $\score^*$.
    
    By plugging the obtained values of $\q_c^*$ and $\score_c^*$ into~\eqref{eq:calibrationfunc:block01loss:breakingSymmetries} and denoting the value $\score_{j}^* + 0.5$ with $\tilde{\score}$, we get
    \begin{align}
    \notag
    \calibrationfunc_{ij}(\eps) = \min_{\tilde{\score}} \; &\frac{1}{2\outputvarcard}  \left( \blocksize_{\blockOfLabel{i}} (\tilde{\score} - \frac{\eps}{2})^2 + \blocksize_{\blockOfLabel{j}} (\tilde{\score} + \frac{\eps}{2})^2 \right),\\
    \label{eq:block01loss:hardConstr:proofSquare:obj2}
    \mbox{s.t.}\; & \tilde{\score} \in [-\frac{\eps}{2}, \frac{\eps}{2}].
    \end{align}
    By setting the derivative of the objective~\eqref{eq:block01loss:hardConstr:proofSquare:obj2} to zero, we get
    $$
    \tilde{\score} = \frac{\eps}{2} \frac{\blocksize_{\blockOfLabel{i}} - \blocksize_{\blockOfLabel{j}}}{\blocksize_{\blockOfLabel{i}} + \blocksize_{\blockOfLabel{j}}},
    $$
    which belongs to the segment $[-\frac{\eps}{2}, \frac{\eps}{2}]$.
    We now compute the function value at this point:
    $$
    \calibrationfunc_{ij} (\eps) = \frac{\eps^2}{4\outputvarcard} \frac{2\blocksize_{\blockOfLabel{i}}\blocksize_{\blockOfLabel{j}}}{\blocksize_{\blockOfLabel{i}}+\blocksize_{\blockOfLabel{j}}},
    $$
    which finishes the proof.
\end{proof}

\subsection{Hamming loss}
\label{sec:exactTranferProofs:hamming}
Recall that $\lossmatrix_{\hamming,\hamminglen}$ is the Hamming loss defined over $\hamminglen$ binary variables (see Eq.~\eqref{eq:hammingLoss} for the precise definition).
In this section, we compute the calibration function for the case of the scores belonging to the column span of the loss matrix (Proposition~\ref{th:calibrationFunction:hammingLoss:hardConstr}).

\begin{proposition}
    \label{th:calibrationFunction:hammingLoss:hardConstr}
    Assume that the scores~$\scorev$ always belong to the column span of the Hamming loss matrix~$\lossmatrix_{\hamming,\hamminglen}$, i.e., $\scoresubset_{\hamming,\hamminglen} = \colspace(\lossmatrix_{\hamming,\hamminglen}) \subseteq \R^k$.
    Then, the calibration function can be computed as follows:
    \[
    \calibrationfunc_{\surrogatelossquad,\lossmatrix_{\hamming,\hamminglen},\scoresubset_{\hamming,\hamminglen}}(\eps)
    =
    \frac{\eps^2}{8\hamminglen}, \quad 0 \leq \eps \leq 1.
    \]
\end{proposition}
\begin{proof}
    We start the proof by applying Lemma~\ref{th:breakingSymmetries} and by studying the vector of the expected losses~$(\lossmatrix_{\hamming,\hamminglen}) \qv$.
    We note that the $\hat{\outputvarv}$-th element~$\lossweighted_{\hat{\outputvarv}}(\qv)$, $\hat{\outputvarv} = (\hat{\outputvar}_t)_{t=1}^\hamminglen$, $\hat{\outputvar}_t \in \{0,1\}$, has a simple form of
    \[
    \lossweighted_{\hat{\outputvarv}}(\qv)
    =
    \sum_{\outputvarv \in \outputdomain} \frac{\q_\outputvarv}{\hamminglen}\sum_{t=1}^\hamminglen [\hat{\outputvar}_t \!\neq\! \outputvar_t]
    =
    1 - \frac{1}{\hamminglen} \sum_{t=1}^\hamminglen  \sum_{\outputvarv \in \outputdomain} \q_\outputvarv [\hat{\outputvar}_t \!=\! \outputvar_t].
    \]
    The quantity~$\sum_{\outputvarv \in \outputdomain} \q_\outputvarv [\hat{\outputvar}_t \!=\! \outputvar_t]$ corresponds to the marginal probability of a variable~$t$ taking a label~$\hat{\outputvar}_t$.
    Note that the expected loss~$\lossweighted_{\hat{\outputvarv}}(\qv)$ only depends on~$\qv$ through marginal probabilities, thus two distributions~$\qv_1$ and $\qv_2$ with the same marginals would be indistinguishable when plugged in the optimization problem for $\calibrationfunc_{ij}(\eps)$~\eqref{eq:calibrationfunc:generalBreakingSymmetries}, given that both the constraints and the objective (by Lemma~\ref{th:quadSurrogateExcess}) only depend on $\qv$ through the expected loss~$\lossweighted_{\hat{\outputvarv}}(\qv)$. 
    Having this in mind, we can consider only separable distributions, i.e., $\q_\outputvarv = \prod_{t=1}^\hamminglen \bigl(q_{t}[\outputvar_t\!=\!1] + (1-q_{t})[\outputvar_t\!=\!0]\bigr)$, where $q_{t} \in [0,1]$, $t=1,\dots,\hamminglen$, are the parameters defining the distribution.
    
    By combining the notation above with Lemmas~\ref{th:quadSurrogateExcess} and~\ref{th:breakingSymmetries}, we arrive at the following optimization problem:
    \begin{align}
    \label{eq:calibrationfunc:hammingloss:breakingSymmetries}
    \calibrationfunc_{\tilde{\outputvarv}\hat{\outputvarv}}(\eps)
    =
    \min_{\scorev, \qv} \;& \frac{1}{2 \outputvarcard}\! \sum_{\outputvarv \in \outputdomain}^{\outputvarcard} \Bigl( \score_\outputvarv \!+\! 1 \!-\! \frac{1}{\hamminglen} \!\!\sum\nolimits_{t=1}^\hamminglen \!\q_{t, \outputvar_t}\Bigr)^2, \\
    \label{eq:calibrationfunc:hammingloss:breakingSymmetries:epsConstr}
    \text{s.t.} \;& \frac{1}{\hamminglen}\!\!\sum\nolimits_{t=1}^{\hamminglen} (\q_{t, \tilde{\outputvar}_t} \!-\! \q_{t, \hat{\outputvar}_t}) \geq \eps,\\
    \label{eq:calibrationfunc:hammingloss:breakingSymmetries:qConstr}
    & \frac{1}{\hamminglen}\!\!\sum\nolimits_{t=1}^{\hamminglen} (\q_{t, \tilde{\outputvar}_t} \!-\! \q_{t, \outputvar_t}) \geq 0,\: \forall \outputvarv \in \outputdomain, \\
    \label{eq:calibrationfunc:hammingloss:breakingSymmetries:fConstr}
    &\score_{\hat{\outputvarv}} \geq \score_\outputvarv,\; \forall \outputvarv \in \outputdomain, \\
    \label{eq:calibrationfunc:hammingloss:breakingSymmetries:simplexConstr}
    & 0\leq \q_t \leq 1, \; t = 1,\dots, \hamminglen,\\
    \label{eq:calibrationfunc:breakingSymmetries:hammingloss:fSet}
    & \scorev \in \scoresubset,
    \end{align}
    where~$q_{t, \outputvar_t}$ is a shortcut to $q_{t}[\outputvar_t\!=\!1] + (1-q_{t})[\outputvar_t\!=\!0]$ and labels $\tilde{\outputvarv}$ and $\hat{\outputvarv}$ serve as the selected labels~$i$ and $j$, respectively.
    
    The calibration function $\calibrationfunc_{\surrogatelossquad,\lossmatrix_{\hamming,\hamminglen},\scoresubset_{\hamming,\hamminglen}}(\eps)=
    \frac{\eps^2}{8\hamminglen}$ in the formulation of this proposition matches the lower bound provided by Theorem~\ref{th:lowerBoundcalibrationFunction} in Section~\ref{sec:bounds:interpretations}.
    Thus, it suffices to construct a feasible w.r.t.~\eqref{eq:calibrationfunc:hammingloss:breakingSymmetries:epsConstr}-\eqref{eq:calibrationfunc:breakingSymmetries:hammingloss:fSet} assignment of variables $\scorev$, $\qv$ and labels~$\tilde{\outputvarv}$, $\hat{\outputvarv}$ such that the objective equals the lower bound.
    
    It suffices to simply set $\tilde{\outputvarv}$ to all zeros and $\hat{\outputvarv}$ to all ones.
    In this case, the constraints~\eqref{eq:calibrationfunc:hammingloss:breakingSymmetries:epsConstr} and~\eqref{eq:calibrationfunc:hammingloss:breakingSymmetries:qConstr} take the simplified form:
    \begin{align}
    \label{eq:calibrationfunc:hammingloss:breakingSymmetries:epsConstr:2}
    \frac{1}{\hamminglen}\!\!\sum\nolimits_{t=1}^{\hamminglen} (1 - 2\q_t) \geq \eps,\\
    \label{eq:calibrationfunc:hammingloss:breakingSymmetries:qConstr:2}
    \q_t \leq \frac12, \; t = 1,\dots, \hamminglen.
    \end{align}
    
    We now set~$\q_t := \frac12 - \frac{\eps}{2}$, $t = 1, \dots, \hamminglen$, and $\scorev := -\frac12\one_{\outputvarcard}$.
    This point is clearly feasible when $0\leq\epsilon\leq 1$, so it remains to compute the value of the objective.
    We complete the proof by writing (let~$w$ be the count of ones in an assignment $\outputvarv$):
    \begin{align*}
    &\frac{1}{2 \outputvarcard}\! \sum_{\outputvarv \in \outputdomain}^{\outputvarcard} \Bigl( \score_\outputvarv \!+\! 1 \!-\! \frac{1}{\hamminglen} \!\!\sum\nolimits_{t=1}^\hamminglen \!\q_{t, \outputvar_t}\Bigr)^2
    =\\
    &\frac{1}{2 \outputvarcard} \sum_{w=0}^{\hamminglen} \tbinom{\hamminglen}{w} \bigl(\frac12 - \frac{1}{\hamminglen} (w(\frac12 - \frac{\eps}{2}) + (\hamminglen - w)(\frac12 + \frac{\eps}{2})) \bigr)^2 =
    \\
    & \frac{1}{2 \outputvarcard} \sum_{w=0}^{\hamminglen} \tbinom{\hamminglen}{w} (\frac{\eps}{2} - \frac{w\eps}{\hamminglen})^2 = \frac{\eps^2}{2 \outputvarcard} \sum_{w=0}^{\hamminglen} \tbinom{\hamminglen}{w} (\frac14 - \frac{w}{\hamminglen} + \frac{w^2}{\hamminglen^2}) =\\
    &\frac{\eps^2}{2 \outputvarcard} (\frac14 2^\hamminglen - \frac{1}{\hamminglen} \hamminglen 2^{\hamminglen-1} + \frac{1}{\hamminglen^2}\hamminglen(\hamminglen+1)2^{\hamminglen-2}) = \frac{\eps^2}{8\hamminglen},
    \end{align*}
    where we use the equality~$\outputvarcard = 2^\hamminglen$ and the identities $\sum_{t=0}^\hamminglen \binom{\hamminglen}{t} = 2^\hamminglen$, $\sum_{t=0}^\hamminglen t \binom{\hamminglen}{t} = \hamminglen 2^{\hamminglen-1}$, $\sum_{t=0}^\hamminglen t^2 \binom{\hamminglen}{t} = \hamminglen (\hamminglen+1) 2^{\hamminglen-2}$.
\end{proof}

\subsection{Mixed 0-1 and block 0-1 loss}
\label{sec:exactTranferProofs:mixedloss}
Recall that $\lossmatrix_{01,\numblocks,\consbreakpoint}$ is the convex combination of the 0-1 loss and the block 0-1 loss with $\numblocks$~blocks, i.e., $\lossmatrix_{01,\numblocks,\consbreakpoint} = \consbreakpoint \lossmatrix_{01} + (1 - \consbreakpoint) \lossmatrix_{01,\numblocks}$, $0 \leq \consbreakpoint \leq 1$.
Let all the blocks be of the same size $\blocksize = \frac{\outputvarcard}{\numblocks} \geq 2$.
In this section, we compute the calibration functions for the case of unconstrained scores (Proposition~\ref{th:calibrationFunction:mixedLoss}) and for the case when scores belong to the column span of the loss matrix (Proposition~\ref{th:calibrationFunction:mixedLoss:hardConstr}).

\begin{proposition}
    \label{th:calibrationFunction:mixedLoss}
    If there are no constraints on scores~$\scorev$ then the calibration function
    \[
    \calibrationfunc_{\surrogatelossquad,\lossmatrix_{01,\numblocks,\consbreakpoint},\R^\outputvarcard}(\eps)
    =
    \begin{cases}
    \frac{\eps^2}{4\outputvarcard}, \qquad \text{$\eps \leq \frac{\consbreakpoint}{1-\consbreakpoint}$}, \\
    \frac{\eps^2 \blocksize}{2\outputvarcard(\blocksize+1)} \!-\! \frac{\consbreakpoint(\eps+1)(\blocksize-1)}{4\outputvarcard(\blocksize+1)}(2\eps \!-\! \eps \consbreakpoint \!-\! \consbreakpoint)
    	\quad \text{$\frac{\consbreakpoint}{1-\consbreakpoint} \leq \eps \leq 1$}
    \end{cases}
    \]
    shows that the surrogate is consistent.
\end{proposition}
Note that when $\consbreakpoint=0$, we have $\calibrationfunc(\eps) = \frac{\eps^2}{4\outputvarcard}\frac{2\blocksize}{\blocksize+1}$ as in Proposition~\ref{th:calibrationFunction:block01loss}.
When $\consbreakpoint \geq 0.5$ we have $\calibrationfunc(\eps) = \frac{\eps^2}{4\outputvarcard}$, which matches Proposition~\ref{th:calibrationFunction:01loss}.

\begin{proof}
    This proof is very similar to the proof of Proposition~\ref{th:calibrationFunction:block01loss}, but technically more involved.
    
    We start by noting that the $i$-th element of the vector~$(\lossmatrix_{01,\numblocks,\consbreakpoint}) \qv$ equals
    \begin{equation}
    \label{eq:calibrationFunction:mixedLoss:expectedLoss}
    \sum_{j:\;\blockOfLabel{j}\neq \blockOfLabel{i}} \!\!\!\!\!\! (1-\consbreakpoint)\q_j\; + \sum_{j:\:j\neq i}\!\! \consbreakpoint \q_j  
    = 
    \consbreakpoint(1-\q_i) + (1-\consbreakpoint)(1-\Q_{\blockOfLabel{i}}),
    \end{equation}
    where for $\blockOfLabel{i}$ and $\Q_v$ we reuse the notation defined in the proof of Proposition~\ref{th:calibrationFunction:block01loss}.
    By combining this with Lemmas~\ref{th:quadSurrogateExcess} and~\ref{th:breakingSymmetries}, we get
    \begin{align}
    \label{eq:calibrationFunction:mixedLoss:proofSquare:obj1}
    \calibrationfunc_{ij}(\eps) \!\!=\!\! \min_{\scorev, \qv} \; & \frac{1}{2\outputvarcard} \sum_{v=1}^\numblocks \;\sum_{c \in\labelsOfBlock{v}} (\score_c + 1 - \consbreakpoint \q_c - (1-\consbreakpoint)\Q_{\blockOfLabel{c}})^2, \\
    \notag \mbox{s.t.} \;& \consbreakpoint (\q_i - \q_j) + (1-\consbreakpoint)(\Q_{\blockOfLabel{i}} - \Q_{\blockOfLabel{j}}) \geq \eps, \\
    \notag & \consbreakpoint (\q_i - \q_c) + (1-\consbreakpoint)(\Q_{\blockOfLabel{i}} - \Q_{\blockOfLabel{c}}) \geq 0, \forall c \\
    \notag & \score_{j} \geq \score_c, \; \forall c, \\
    \notag & \sum_{c=1}^\outputvarcard \q_c = 1, \quad \q_c \geq 0, \; \forall c.
    \end{align}
    
    The blocks are all of the same size so we need to consider just the two cases: 1) the selected labels belong to the same block, i.e., $\blockOfLabel{i} = \blockOfLabel{j}$; 2) the selected labels belong to the two different blocks, i.e., $\blockOfLabel{i} \neq \blockOfLabel{j}$.
    
    The first case can be proven by a straight forward generalization of the proof of Proposition~\ref{th:calibrationFunction:01loss}.
    Given that the loss value is bounded by~$1$, the maximal possible value of~$\eps$ when the constraints can be feasible equals~$\consbreakpoint$.
    Thus, we have $\calibrationfunc_{ij}(\eps) = \frac{\eps^2}{4\outputvarcard}$ for $\eps \leq \consbreakpoint$ and $+\infty$ otherwise.
    
    We will now proceed to the second case $\blockOfLabel{i} \neq \blockOfLabel{j}$.
    We show that
    \[
    \calibrationfunc_{ij}(\eps) \!\!=\!\! \begin{cases}
    \frac{\eps^2}{4\outputvarcard}, \quad \text{for $\eps \leq \frac{\consbreakpoint}{1-\consbreakpoint}$}, \\
    \frac{\eps^2 \blocksize}{2\outputvarcard(\blocksize+1)} - \frac{\consbreakpoint(\eps+1)(\blocksize-1)}{4\outputvarcard(\blocksize+1)}(2\eps - \eps \consbreakpoint - \consbreakpoint), \text{otherwise.} \\
    \end{cases}
    \]
    
    Similarly to the arguments used in Propositions~\ref{th:calibrationFunction:01loss} and~\ref{th:calibrationFunction:block01loss}, we claim that there is an optimal point of~\eqref{eq:calibrationFunction:mixedLoss:proofSquare:obj1} such that $\q^*_c=0$, $c \not\in \{i,j\}$; $\q^*_i = 0.5 + \frac{\eps}{2}$;  $\q^*_{j} = 0.5 - \frac{\eps}{2}$; and $\score^*_c = -1$ for $c \not\in \labelsOfBlock{ij} := \labelsOfBlock{\blockOfLabel{i}} \cup \labelsOfBlock{\blockOfLabel{j}}$.
    
    First, we will show that we can consider only configurations with all the probability mass on the two selected blocks $\blockOfLabel{i}$ and $\blockOfLabel{j}$.
    Given any optimal point $\scorev^*$ and $\qv^*$, the operation (with $\delta = \sum_{c \not\in \labelsOfBlock{ij}} \q_c^*$)
    \begin{align*}
    \score_i^* &:= \score_i^* + \frac{\delta}{2}, &
    &\q_i^* := \q_i^* + \frac{\delta}{2}, \\
    \score_j^* &:= \score_j^* + \frac{\delta}{2}, &
    &\q_{j}^* := \q_{j}^* + \frac{\delta}{2},\\
    \score_c^* &:= -1, c \not\in\labelsOfBlock{ij} &
    &\q_c^* := 0, c \not\in\labelsOfBlock{ij} \\
    \score_c^* &:= \score_c^* + (1-\consbreakpoint)\frac{\delta}{2}, \mathrlap{\;\;c \in \;\labelsOfBlock{ij} \setminus \{i,j\}}
    \end{align*}
    can only decrease the objective of~\eqref{eq:calibrationFunction:mixedLoss:proofSquare:obj1} because the summands corresponding to the unselected $\numblocks-2$ blocks are set to zero.
    All the constraints stay feasible and the values corresponding to the blocks $\blockOfLabel{i}$ and $\blockOfLabel{j}$ do not change.
    The last operation is required, because the values~$\Q_{\blockOfLabel{i}}$, $\Q_{\blockOfLabel{j}}$ change when we change $\q_i$ and $\q_j$.
    Adding $(1-\consbreakpoint)\frac{\delta}{2}$ to some scores compensates this and cannot violate the constraints because $\score_j^*$ goes up by $\frac{\delta}{2} \geq (1-\consbreakpoint)\frac{\delta}{2}$.
    
    Now we will show that it is possible to move all the mass to the two selected labels $i$ and $j$.
    We cannot simply move the mass within one block, but need to create some overflow and move it to another block in a specific way.
    Consider $\delta:=q_a^*$, which is some non-zero mass on a non-selected label of the block~$\blockOfLabel{i}$.
    Then, the operation
    \begin{align*}
    \score_i^* &:= \score_i^* + \delta\frac{\consbreakpoint}{2}, &
    \q_i^* &:= \q_i^* + \delta(1-\frac{\consbreakpoint}{2}), \\
    \score_j^* &:= \score_j^* + \delta\frac{\consbreakpoint}{2}, &
    \q_{j}^* &:= \q_{j}^* + \delta\frac{\consbreakpoint}{2},\\
    \score_a^* &:= \score_a^* + \delta\frac{\consbreakpoint}{2}(\consbreakpoint - 3), &
    \q_{a}^* &:= \q_{a}^* - \delta = 0,\\
    \score_c^* &:= \score_c^* - \delta\frac{\consbreakpoint}{2} (1-\consbreakpoint), \mathrlap{\;\;c \in \;\labelsOfBlock{i} \setminus \{i,a\}} &&\\
    \score_c^* &:= \score_c^* +  \delta\frac{\consbreakpoint}{2}(1-\consbreakpoint), \mathrlap{\;\;c \in \;\labelsOfBlock{j} \setminus \{j\}} &&
    \end{align*}
    does no change the objective value of~\eqref{eq:calibrationFunction:mixedLoss:proofSquare:obj1} because the quantities $\score_c + 1 -  \consbreakpoint \q_c - (1-\consbreakpoint)\Q_{\blockOfLabel{c}}$, $c \in \labelsOfBlock{ij}$, stay constant and all the constraints of~\eqref{eq:calibrationFunction:mixedLoss:proofSquare:obj1} stay feasible.
    We repeat this operation for all $a \in \labelsOfBlock{\blockOfLabel{i}}\setminus \{i\}$ and, thus, move all the probability mass within the block~$\blockOfLabel{i}$ to the label~$i$.
    In the block~$\blockOfLabel{j}$, an analogous operation can move all the mass to the label~$j$.
    
    It remains to show that $\q_i^* - \q_j^* = \eps$.
    Indeed, if $\q^*_i - \q^*_j = \delta' >  \eps$, the operation analogous to~\eqref{eq:tfproofs:truncatingToEps}
    \begin{align*}
    \score_i^* &:= \score_i^* - \frac{\delta'-\eps}{2}, &
    \q_i^* &:= \q_i^* - \frac{\delta'-\eps}{2}, \\
    \score_j^* &:= \score_j^* + \frac{\delta'-\eps}{2},&
    \q_{j}^* &:= \q_{j}^* + \frac{\delta'-\eps}{2}, \\
    \score_c^* &:= \score_c^* - (1-\consbreakpoint)\frac{\delta'-\eps}{2}, \mathrlap{c \in  \labelsOfBlock{\blockOfLabel{i}}\setminus \{i\},} \\
    \score_c^* &:= \score_c^* + (1-\consbreakpoint)\frac{\delta'-\eps}{2}, \mathrlap{c \in \labelsOfBlock{\blockOfLabel{j}}\setminus \{j\}}
    \end{align*}
    can always set  $\q_i^* - \q_j^* = \eps$, and thus  $\q^*_i = 0.5 + \frac{\eps}{2}$ and $\q^*_{j} = 0.5 - \frac{\eps}{2}$.
    After this operation, all the scores of the block~$\blockOfLabel{i}$ go down and all the scores of the block~$\blockOfLabel{j}$ go up at most as much as~$\score_j^*$, so the constraints $\score_{j} \geq \score_c$ cannot get violated.
    
    We now proceed with the computation of $\calibrationfunc_{ij}(\eps)$.
    First, we note that convexity and symmetries of~\eqref{eq:calibrationFunction:mixedLoss:proofSquare:obj1} implies that all the non-selected scores within each block are equal.\footnote{If these optimal scores are not equal, by symmetry, one can obtain the same objective and feasibility by permuting their corresponding values. By taking a uniform convex combination on all permutations, we obtain a point where all the scores are equal, and by convexity, would yield a lower objective value.}
    Denote the scores of the non-selected labels of the block~$\blockOfLabel{i}$ by~$\score_i'$, and the scores of the non-selected labels of the block~$\blockOfLabel{j}$ by~$\score_j'$.
    
    Analogous to all the previous propositions, the truncation argument gives us that all the values $\score_c^*$ belong to the segment~$[-1, -0.5 + \frac{\eps}{2}]$.
    For all the optimal values $\score_c^*$, $c \neq j$, the following identity holds:
    \[
    \score_c^* = \begin{cases}
    \score_j^*, \quad \text{if\;\; $\consbreakpoint \q_c^* + (1-\consbreakpoint)\Q_{\blockOfLabel{c}}^*-1 \geq \score_j^*$},\\
    \consbreakpoint \q_c^* + (1-\consbreakpoint)\Q_{\blockOfLabel{c}}^*-1, \quad \text{otherwise}.
    \end{cases}
    \]
    Given that $\score_i^*$ wants to equal the maximal possible value $-0.5 + \frac{\eps}{2}$, it implies that $\score_i^* = \score_j^*$.
    Denote this value by~$\score$.
    
    By, plugging the values of~$\qv^*$ and $\scorev^*$ provided above into the objective of~\eqref{eq:calibrationFunction:mixedLoss:proofSquare:obj1}, we get
    \begin{align}
    \notag
    \frac{1}{2\outputvarcard}\Bigl(  
    &(\score \!+\! 0.5 \!-\! \frac{\eps}{2})^2
    \!+\!
    (\blocksize\!-\!1)(\score_i' \!+\! 1\!-\!(1\!-\!\consbreakpoint)(0.5\!+\!\frac{\eps}{2}))^2
    + \\
    \label{eq:calibrationFunction:mixedLoss:proofSquare:case2:obj2}
    &(\score \!+\! 0.5 \!+\! \frac{\eps}{2})^2
    \!+\!
    (\blocksize\!-\!1)(\score_j' \!+\! 1\!-\!(1\!-\!\consbreakpoint)(0.5\!-\!\frac{\eps}{2}))^2
    \Bigr).
    \end{align}
    By minimizing~\eqref{eq:calibrationFunction:mixedLoss:proofSquare:case2:obj2} without constraints, we get $\score^*=-0.5$, $\score_i'^*=\frac12(1+\eps)(1-\consbreakpoint)-1$, $\score_j'^*=\frac12(1-\eps)(1-\consbreakpoint)-1$.
    We now need to compare $\score_i'^*$ and $\score_j'^*$ with $\score^*$ to satisfy the constraints $\score^* \geq \score_i'^*$ and $\score^* \geq \score_j'^*$.
    First, we have that
    \[
    \score^*\!-\!\score_j'^* = \frac12(\consbreakpoint+\eps-\consbreakpoint\eps) \geq 0,\; \text{for $0\!\leq\!\eps\!\leq\!1$ and $0\!\leq\!\consbreakpoint\!\leq\!1$}.
    \]
    Second, we have 
    \[
    \score^*-\! \score_i'^* \!=\! \frac12(\consbreakpoint-\eps+\consbreakpoint\eps) \geq 0,\; \text{for $0\!\leq\!\eps\!\leq\!\frac{\consbreakpoint}{1-\consbreakpoint}$ and $0\!\leq\!\consbreakpoint\!\leq\! 1$}.
    \]
    We can now conclude that when $\eps\leq\frac{\consbreakpoint}{1-\consbreakpoint}$ we have both $\score_i'$ and $\score_j'$ equal to their unconstrained minimum points leading to $\calibrationfunc_{ij}(\eps)=\frac{\eps^2}{4\outputvarcard}$.
    
    Now, consider the case $\eps>\frac{\consbreakpoint}{1-\consbreakpoint}$.
    We have the constraint $\score \geq \score_i'$ violated, so at the minimum we have $\score_i'= \score$. The new unconstrained minimum w.r.t.~$\score$ equals $\score^* = \frac{1}{\blocksize+1}(-1-(\blocksize-1)(1- \frac12(1-\consbreakpoint)(1-\eps)))$.
    We now show that the inequality $\score^* \geq \score_j'^*$ still holds. 
    We have
    \[
    \score^* - \score_j'^* = \frac{\consbreakpoint+\eps \blocksize - \consbreakpoint \eps \blocksize}{\blocksize+1} \geq 0, \; \text{for \:$0\leq\eps\leq 1$ \:and \:$0\leq \consbreakpoint \leq 1$}.
    \]
    Substitution of $\score_i'^*=\score^*$ and $\score_j'^*$ into~\eqref{eq:calibrationFunction:mixedLoss:proofSquare:case2:obj2} gives us 
    \[
    \frac{1}{\outputvarcard} \Bigl( \frac{\eps^2 \blocksize}{2(\blocksize+1)} - \frac{\consbreakpoint(\eps+1)(\blocksize-1)}{4(\blocksize+1)}(2\eps - \eps \consbreakpoint - \consbreakpoint)\Bigr),
    \]
    which equals $\calibrationfunc_{ij}(\eps)$ for $1\geq \eps>\frac{\consbreakpoint}{1-\consbreakpoint}$.
    
    Comparing cases~1 and~2, we observe that $\calibrationfunc_{ij}(\eps)$ from case~2 is never larger than the one of case~1, thus case~2 provides the overall calibration function~$\calibrationfunc_{ij}(\eps)$.
\end{proof}

\begin{proposition}
    \label{th:calibrationFunction:mixedLoss:hardConstr}
    If the scores~$\scorev$ are constrained to be equal inside the blocks, i.e.\ belong to the subspace~$\scoresubset_{01,\numblocks} = \colspace(\lossmatrix_{01,\numblocks}) \subseteq \R^k$, then the calibration function
    \[
    \calibrationfunc_{\surrogatelossquad,\lossmatrix_{01,\numblocks,\consbreakpoint},\scoresubset_{01,\numblocks}}(\eps)
    =
    \begin{cases}
    \frac{(\eps-\frac{\consbreakpoint}{2})^2}{4\numblocks}\frac{(\frac{\consbreakpoint \numblocks}{\outputvarcard}+1-\consbreakpoint)^2}{(1-\frac{\consbreakpoint}{2})^2}, \;&\text{$\frac{\consbreakpoint}{2} \leq \eps \leq 1$}, \\
    0,\;&\text{$0\leq \eps \leq \frac{\consbreakpoint}{2}$}
    \end{cases}
    \]
    shows that the surrogate is consistent up to level~$\frac{\consbreakpoint}{2}$.
\end{proposition}
When $\consbreakpoint=0$, we have $H(\eps) = \frac{\eps^2}{4 \numblocks}$ as in Proposition~\ref{th:calibrationFunction:block01loss:hardConstr}.
When $\consbreakpoint > 0$ we have $H(\eps) = 0$ for small~$\eps$, which corresponds to the case of inconsistent surrogate (0-1 loss and constrained scores).
\begin{proof}
    This proof combines ideas from Proposition~\ref{th:calibrationFunction:mixedLoss} and Proposition~\ref{th:calibrationFunction:block01loss:hardConstr}.
    
    Note that contrary to all the previous results, Lemma~\ref{th:quadSurrogateExcess} is not applicable, because, for $\numblocks < \outputvarcard$, we have that $\colspace(\lossmatrix_{01,\numblocks,\consbreakpoint}) = \R^\outputvarcard \not\subset \scoresubset_{01,\numblocks} = \colspace(\lossmatrix_{01,\numblocks})$.
    
    We now derive an analog of Lemma~\ref{th:quadSurrogateExcess} for this specific case.
    We define the subspace of scores~$\scoresubset_{01,\numblocks} = \{ \scorematrix \scoreparamv \mid \scoreparamv \in \R^{\numblocks} \}$ with a matrix $\scorematrix : = \scorematrix_{01,\numblocks} \in \R^{\outputvarcard \times \numblocks}$ with columns containing the indicator vectors of the blocks.
    We have $\scorematrix^\transpose \scorematrix = \blocksize \id_{\numblocks}$ and thus $(\scorematrix^\transpose \scorematrix)^{-1} = \frac{1}{\blocksize} \id_{\numblocks}$.
    We shortcut the loss matrix $\lossmatrix_{01,\numblocks,\consbreakpoint}$ to $\lossmatrix$ and rewrite it as
    \[
    \lossmatrix = \consbreakpoint \lossmatrix_{01} + (1-\consbreakpoint)\lossmatrix_{01,\numblocks} =
    \one_\outputvarcard \one_\outputvarcard^\transpose - \consbreakpoint \id_\outputvarcard - (1-\consbreakpoint)\scorematrix\scorematrix^\transpose.
    \]
    By redoing the derivation of Lemma~\ref{th:quadSurrogateExcess}, we arrive at a different excess surrogate:
    \begin{align*}
    \surrogateweighted(\scorev(\scoreparamv), \qv) &= \frac{1}{2\outputvarcard}(\blocksize \scoreparamv^\transpose \scoreparamv + 2 \scoreparamv^\transpose \scorematrix^\transpose \lossmatrix \qv) + r(\qv), \\
    \scoreparamv^* &:= \argmin\nolimits_\scoreparamv \surrogateweighted(\scorev(\scoreparamv), \qv) =  -\frac{1}{\blocksize} \scorematrix^\transpose \lossmatrix \qv, \\
    \excess\surrogateweighted(\scorev(\scoreparamv), \qv) &= \frac{1}{2\outputvarcard}(\blocksize \scoreparamv^\transpose \scoreparamv + 2 \scoreparamv^\transpose \scorematrix^\transpose \lossmatrix \qv + \frac{1}{\blocksize} \qv^\transpose \lossmatrix^\transpose \scorematrix \scorematrix^\transpose \lossmatrix \qv)\\
    &=
    \frac{\blocksize}{2\outputvarcard} \| \scoreparamv +  \frac{1}{\blocksize} \scorematrix^\transpose \lossmatrix \qv\|_2^2\\
    &= 
    \frac{\blocksize}{2\outputvarcard} \sum_{v=1}^\blocksize
    (\scoreparam_v + 1 - (1-\consbreakpoint) \Q_v - \frac{\consbreakpoint}{\blocksize} \Q_v)^2,
    \end{align*}
    where $\Q_v = \sum_{c \in \labelsOfBlock{v}} \q_c$ is the total probability mass on block~$v$ and $\labelsOfBlock{v} \subset \outputdomain$ denotes the set of labels of block~$v$.
    
    Analogously to Proposition~\ref{th:calibrationFunction:mixedLoss} we can now apply Lemma~\ref{th:breakingSymmetries} and obtain $\calibrationfunc_{ij}(\eps)$.
    \begin{align}
    \label{eq:calibrationFunction:mixedLoss:hardConstr:proofSquare:obj1}
    \calibrationfunc_{ij}(\eps) \!\!=\!\! \min_{\scoreparamv, \qv} \; & \frac{\blocksize}{2\outputvarcard} \sum_{v=1}^\numblocks (\scoreparam_v + 1 - (1\!-\!\consbreakpoint) \Q_v - \frac{\consbreakpoint}{\blocksize} \Q_v)^2, \\
    \notag \mbox{s.t.} \;& \consbreakpoint (\q_i - \q_j) + (1-\consbreakpoint)(\Q_{\blockOfLabel{i}} - \Q_{\blockOfLabel{j}}) \geq \eps, \\
    \notag & \consbreakpoint (\q_i - \q_c) + (1-\consbreakpoint)(\Q_{\blockOfLabel{i}} - \Q_{\blockOfLabel{c}}) \geq 0, \forall c \\
    \notag & \scoreparam_{\blockOfLabel{j}} \geq \scoreparam_u, \; \forall u = 1,\dots,\numblocks, \\
    \notag & \sum_{c=1}^\outputvarcard \q_c = 1, \quad \q_c \geq 0, \; \forall c.
    \end{align}
    The main difference to~\eqref{eq:calibrationFunction:mixedLoss:proofSquare:obj1} consists in the fact that we now minimize w.r.t.\ $\scoreparamv$ instead of~$\scorev$.
    
    Note that because of the way the predictor~$\predictor(\scorev(\scoreparamv))$ resolves ties (among the labels with maximal scores it always picks the label with the smallest index), not all labels can be predicted.
    Specifically, only one label from each block can be picked.
    This argument allows us to assume that $\blockOfLabel{i} \neq \blockOfLabel{j}$ in the remainder of this proof.
    
    First, let us prove the case for $\eps \leq \frac{\consbreakpoint}{2}$.
    We explicitly provide a feasible assignment of variables where the objective equals zero.
    We set $\q_i = \frac12$ and $\q_c = \frac{1}{2(\blocksize-1)}$, $c \in \labelsOfBlock{\blockOfLabel{j}} \setminus \{j\}$.
    All the other labels (including~$j$ and the unselected labels of the block~$\blockOfLabel{i}$) receive zero probability mass.
    This assignment of~$\qv$ implies~$\Q_{\blockOfLabel{i}} = \Q_{\blockOfLabel{j}} = \frac12$ and the zero mass on the other blocks.
    We also set $\scoreparam_{\blockOfLabel{i}}$ and $\scoreparam_{\blockOfLabel{j}}$ to $(1\!-\!\consbreakpoint) \frac12 + \frac{\consbreakpoint}{\blocksize} \frac12 - 1$ to ensure zero objective value.
    Verifying other feasibility constraints we have 
    $\consbreakpoint (\q_i - \q_j) + (1-\consbreakpoint)(\Q_{\blockOfLabel{i}} - \Q_{\blockOfLabel{j}}) = \frac{\consbreakpoint}{2} \geq \eps$ and $\consbreakpoint (\q_i - \q_c) + (1-\consbreakpoint)(\Q_{\blockOfLabel{i}} - \Q_{\blockOfLabel{c}}) = \consbreakpoint(\frac12-\frac{1}{2(\blocksize-1)}) \geq 0$, $c \in \labelsOfBlock{\blockOfLabel{j}} \setminus \{j\}$.
    Other constraints are trivially satisfied.
    
    Now, consider the case of $\eps > \frac{\consbreakpoint}{2}$.
    As usual, we claim  the following values of the variables~$\scorev$ and~$\qv$ result in an optimal point. We have $\q_c^*=0$, $c\not\in \labelsOfBlock{ij}$; $\scoreparam_v^* = -1$, $v \not\in\{\blockOfLabel{i}, \blockOfLabel{j}\}$; and $\q_i^* =\Q_{\blockOfLabel{i}}^* = \frac{1+\eps-\consbreakpoint}{2-\consbreakpoint}$; $\q_c^*=0$, $c \in \labelsOfBlock{\blockOfLabel{i}} \setminus \{i\}$ (other labels in the block~$\blockOfLabel{i}$);  $\q_{j}^* = 0$, $\q_c^* = \frac{1-\eps}{(2-\consbreakpoint)(\blocksize-1)}$, $c \in \labelsOfBlock{\blockOfLabel{j}} \setminus \{j\}$ (other labels in the block~$\blockOfLabel{j}$).
    
    First, we will show that we can consider only configurations with all the probability mass on the two selected blocks~$\blockOfLabel{i}$ and $\blockOfLabel{j}$.
    Given some optimal variables $\scorev^*$ and $\qv^*$, the operation (with $\delta = \sum_{c \in \outputdomain \setminus \labelsOfBlock{ij}} \q_c^*$)
    \begin{align*}
    \q_c^* &:= 0, \; c \in \outputdomain \setminus \labelsOfBlock{ij}, &
    \q_i^* &:= \q_i^* + \frac{\delta}{2}, &
    \q_{j}^* &:= \q_{j}^* + \frac{\delta}{2}, \\
    \scoreparam_v^* &:= -1, \; \mathrlap{v \notin \{\blockOfLabel{i}, \blockOfLabel{j}\},} \\
    \scoreparam_{\blockOfLabel{i}}^* &:= \scoreparam_{\blockOfLabel{i}}^* + \mathrlap{\frac{\delta}{2}(1-\consbreakpoint+\frac{\consbreakpoint}{\blocksize}),} \\
    \scoreparam_{\blockOfLabel{j}}^* & := \scoreparam_{\blockOfLabel{j}}^* + \mathrlap{\frac{\delta}{2}(1-\consbreakpoint+\frac{\consbreakpoint}{\blocksize})}
    \end{align*}
    can only decrease the objective of~\eqref{eq:calibrationFunction:mixedLoss:hardConstr:proofSquare:obj1} because the summands corresponding to the unselected $\numblocks - 2$ blocks are set to zero.
    All the constraints stay feasible and the values corresponding to the blocks $\blockOfLabel{i}$ and $\blockOfLabel{j}$ do not change.
    
    Now, we move the mass within the two selected blocks.
    To start with, moving the mass within one block does not change the objective, because it depends only on $\Q_{\blockOfLabel{c}}$ and not on $\qv$ directly.
    In the block~$\blockOfLabel{i}$, it is safe to increase $\q_i$ and decrease the mass on the other labels, because $\q_i$ enters the constraints with the positive sign and while the others enter with the negative sign. So we let $\q_c = 0$ for $c \in \labelsOfBlock{\blockOfLabel{i}}  / \{i \}$ and $ \Q_{\blockOfLabel{i}}=\q_i$. We also have $\Q_{\blockOfLabel{j}}=1-\q_i$ as the mass on all other blocks is zero.
    
    Moving mass within the block~$\blockOfLabel{j}$ is more complicated, as moving mass to some label~$c$ of this block might violate the constraints of \eqref{eq:calibrationFunction:mixedLoss:hardConstr:proofSquare:obj1} 
    on $\q_i$. We start by considering the first constraint in~$\eqref{eq:calibrationFunction:mixedLoss:hardConstr:proofSquare:obj1}$, using $\Q_{\blockOfLabel{j}}=1-\q_i$, we get:
    \begin{equation}
    \label{eq:mixedLoss:firstConstraint}
    \q_i \geq \eps + \consbreakpoint \q_j + (1-\consbreakpoint)(1-\q_i).
    \end{equation}
    By using $\q_j \geq 0$ and $\eps \geq \frac{\consbreakpoint}{2}$, the inequality~\eqref{eq:mixedLoss:firstConstraint} implies that $\q_i \geq \frac{1}{2}$ and thus that 
    \begin{equation}
    \label{eq:mixedLossQcProperty}
    \q_c \leq \Q_{\blockOfLabel{j}} \leq \frac{1}{2} \quad \forall c \in  \labelsOfBlock{\blockOfLabel{j}} \, .
    \end{equation}
    Now the second constraint of~\eqref{eq:calibrationFunction:mixedLoss:hardConstr:proofSquare:obj1} that we want to satisfy is:
    \begin{equation}
        \label{eq:mixedLoss:secondConstraint}
        \q_i \geq \consbreakpoint \q_c + (1-\consbreakpoint)\Q_{\blockOfLabel{j}} \quad \forall c \in  \labelsOfBlock{\blockOfLabel{j}} \, .
    \end{equation}
    Using~\eqref{eq:mixedLossQcProperty}, we have that the RHS of~\eqref{eq:mixedLoss:secondConstraint} is $\leq 1/2$, and so since $\q_i \geq 1/2$, we have that~\eqref{eq:mixedLoss:secondConstraint} is satisfied for any valid mass distribution on block~$\blockOfLabel{j}$ (i.e. such that $\Q_{\blockOfLabel{j}} \leq 1/2$). Using $\q_j = 0$ gives the most possibilities for the value of~$\q_i$ in the constraint~\eqref{eq:mixedLoss:firstConstraint}. Moreover, the constraint~\eqref{eq:mixedLoss:firstConstraint} is more stringent than the constraint~\eqref{eq:mixedLoss:secondConstraint}, i.e. if it is satisfied, the second one is also satisfied; so we focus only on the first constraint.
    
    As in the proof of all other propositions, we can make the constraint~\eqref{eq:mixedLoss:firstConstraint} an equality for the optimum by generalizing the transformation of~\eqref{eq:tfproofs:truncatingToEps} which makes the constraint tight without changing the objective and maintaining feasibility. So~\eqref{eq:mixedLoss:firstConstraint} as an equality with $\q_j = 0$ yields the value
    $$
    \q_i^* = \frac{1+\eps-\consbreakpoint}{2-\consbreakpoint}.
    $$
    So to summarize at this point, we have $\q_j^* = 0$; $\q_c^* = 0$, $c \in \labelsOfBlock{\blockOfLabel{i}} \setminus \{i\}$; $\q_c^* = 0$, $\labelsOfBlock{\blockOfLabel{i}} \not\in \{\blockOfLabel{i}, \blockOfLabel{j}\}$.  $\q_i^* = \frac{1+\eps-\consbreakpoint}{2-\consbreakpoint}$ and $\Q_{\blockOfLabel{j}}=1-\q_i^*$. The precise distribution of mass for $c \in \labelsOfBlock{\blockOfLabel{j}}/\{j\}$ does not matter (any distribution is feasible and does not influence the objective, only the total mass matters), but for concreteness, we can choose them to all have the same mass yielding $\q_c^* = \frac{1-\eps}{(2-\consbreakpoint)(\blocksize-1)}$, $c \in \labelsOfBlock{\blockOfLabel{j}} \setminus \{j\}$.
    
    We now finish the computation of $\calibrationfunc_{ij}(\eps)$.
    First, we note that, due to the truncation argument similar to the one mentioned in the paragraph after~\eqref{eq:tfproofs:truncatingToEps}, we have  both $\scoreparam_i^*$ and $\scoreparam_j^*$ in the segment $[(1-\consbreakpoint)\Q_{\blockOfLabel{j}}^* + \frac{\consbreakpoint}{\blocksize} \Q_{\blockOfLabel{j}}^*-1, (1-\consbreakpoint)\Q_{\blockOfLabel{i}}^* + \frac{\consbreakpoint}{\blocksize} \Q_{\blockOfLabel{i}}^*-1]$ and since $\scoreparam_j^* \geq \scoreparam_i^*$, we have $\scoreparam_j^* = \scoreparam_i^* =: \scoreparam$ at the optimum.
    
    Substituting the values~$\Q_{\blockOfLabel{i}}^*$ and $\Q_{\blockOfLabel{j}}^*$ provided above into the objective of~\eqref{eq:calibrationFunction:mixedLoss:hardConstr:proofSquare:obj1} and performing unconstrained minimization w.r.t. $\scoreparam$ (we use the help of MATLAB symbolic toolbox to set the derivative to zero) we get
    \[
    \scoreparam^* = -\frac{\blocksize-\consbreakpoint+\consbreakpoint \blocksize}{2\blocksize}
    \]
    and, consequently, 
    \[
    \calibrationfunc_{ij}(\eps) =
    \frac{\blocksize(\eps-\frac{\consbreakpoint}{2})^2(\frac{\consbreakpoint}{\blocksize}+1-\consbreakpoint)^2}{4\outputvarcard(1-\frac{\consbreakpoint}{2})^2},
    \]
    which finishes the proof.
\end{proof}

\section{Constants in the SGD rate}
\label{sec:tranferAndSgd}
To formalize the learning difficulty by bounding the required number of iterations to get a good value of the risk (Theorem~\ref{th:calibrationSGD:kernels}), we need to bound the constants~$D$ and~$M$.
In this section, we provide a way to bound these constants for the quadratic surrogate~$\surrogatelossquad$~\eqref{eq:quadrLoss} under a simplifying assumption slightly stronger than the well-specified model Assumption~\ref{th:well_specification}.

Consider the family of score functions~$\scorefuncset_{\scorematrix, \hilbertspace}$ defined via an explicit feature map~$\featuremap(\inputvarv) \in \hilbertspace$, i.e., $\scorefunc_\parammatrix(\inputvarv) = \scorematrix \parammatrix \featuremap(\inputvarv)$, where a matrix $\scorematrix \in \R^{\outputvarcard \times \scoresubspacedim}$ defines the structure and an operator (which we think of as a matrix with one dimension being infinite) $\parammatrix : \hilbertspace \to \R^{\scoresubspacedim}$ contains the learnable parameters.
Then the surrogate risk can be written as
\[
\risk_\surrogateloss(\scorefunc_\parammatrix)
=
\E_{(\inputvarv, \outputvarv) \sim \data} \frac{1}{2\outputvarcard} \| \scorematrix \parammatrix \featuremap(\inputvarv) + \lossmatrix(:, \outputvarv) \|^2_{\R^\outputvarcard}
\]
and its stochastic w.r.t. $(\inputvarv, \outputvarv)$ gradient as
\begin{equation}
\gv_{\inputvarv, \outputvarv}(\parammatrix) = \frac{1}{\outputvarcard} \scorematrix^\transpose (\scorematrix \parammatrix \featuremap(\inputvarv) + \lossmatrix(:, \outputvarv)) \featuremap(\inputvarv)^\transpose
\end{equation}
where $\lossmatrix(:, \outputvarv)$ denotes the column of the loss matrix corresponding to the correct label~$\outputvarv$.
Note that computing the stochastic gradient requires performing products $\scorematrix^\transpose \scorematrix$ and $\scorematrix^\transpose \lossmatrix(:, \outputvarv)$ for which direct computation is intractable when $\outputvarcard$ is exponential, but which can be done in closed form for the structured losses we consider (the Hamming and block 0-1 loss).
More generally, these operations require suitable inference algorithms.

To derive the constants, we use a simplifying assumption stronger than Assumption~\ref{th:well_specification} in the case of quadratic surrogate: we assume that the conditional~$\q_c(\inputvarv)$, seen as a function of $\inputvarv$, belongs to the RKHS~$\hilbertspace$, which by the reproducing property implies that for each~$c=1,\dots,\outputvarcard$, there exists~$v_c \in \hilbertspace$ such that $\q_c(\inputvarv) = \langle v_c, \featuremap(\inputvarv) \rangle_\hilbertspace$ for all~$\inputvarv\in\inputdomain$.
Concatenating all~$v_c$, we get an operator $V: \hilbertspace \to \R^\outputvarcard$.
To derive the bound, we also assume that $\sum_{c=1}^{\outputvarcard} \|v_c\|_\hilbertspace \leq \qbound$ and $\|\featuremap(\inputvarv)\|_{\hilbertspace} \leq \rkhsbound$ for all~$\inputvarv \in \inputdomain$. In the following, we use the notation $\qv_\inputvarv$ to denote the vector in $\R^\outputvarcard$ with components $\q_c(\inputvarv)$, $c=1,\dots,\outputvarcard$, for a fixed $\inputvarv$, and thus $\qv_\inputvarv = V\featuremap(\inputvarv)$.

Under these assumptions, we can write the theoretical minimum of the surrogate risk.
The gradient of the surrogate risk gives
\begin{align*}
\outputvarcard \gradient_{\parammatrix} \risk_\surrogateloss(\scorefunc_\parammatrix)
&=
\scorematrix^\transpose \scorematrix \parammatrix \E_{\inputvarv \sim \data_{\inputdomain}}(\featuremap(\inputvarv) \featuremap(\inputvarv)^\transpose) + \scorematrix^\transpose \lossmatrix \E_{\inputvarv \sim \data_{\inputdomain}}(\qv_\inputvarv \featuremap(\inputvarv)^\transpose ) \\
&=
\scorematrix^\transpose \scorematrix \parammatrix \E_{\inputvarv \sim \data_{\inputdomain}}(\featuremap(\inputvarv) \featuremap(\inputvarv)^\transpose) + \scorematrix^\transpose \lossmatrix V \E_{\inputvarv \sim \data_{\inputdomain}}(\featuremap(\inputvarv) \featuremap(\inputvarv)^\transpose) \\
&= 
\left( \scorematrix^\transpose \scorematrix \parammatrix + \scorematrix^\transpose \lossmatrix V \right)  \E_{\inputvarv \sim \data_{\inputdomain}}(\featuremap(\inputvarv) \featuremap(\inputvarv)^\transpose) .
\end{align*}
Setting the content of the parenthesis to zero gives that $\parammatrix^* = -(\scorematrix^\transpose \scorematrix)^\pinv \scorematrix^\transpose \lossmatrix V$ is a solution to the stationary condition equation $\gradient_{\parammatrix} \risk_\surrogateloss(\scorefunc_\parammatrix) = 0$. 

We can now bound the Hilbert-Schmidt norm of this choice of optimal parameters~$\parammatrix^*$ as
\begin{align*}
\|\parammatrix^*\|_{HS}
&=
\| (\scorematrix^\transpose \scorematrix)^\pinv \scorematrix^\transpose \lossmatrix V \|_{HS}
&&
\\
&\leq
\| (\scorematrix^\transpose \scorematrix)^\pinv \scorematrix^\transpose \|_{HS} \|  \lossmatrix V \|_{HS} 
&&
\quad\text{//submultiplicativity of $\|\cdot\|_{HS}$}
\\
&\leq
\sqrt{\scoresubspacedim} \| (\scorematrix^\transpose \scorematrix)^\pinv \scorematrix^\transpose \|_2 \|  \lossmatrix V \|_{HS}
&&
\quad\text{//connection of $\|\cdot\|_{HS}$ and $\|\cdot\|_2$ via $\scoresubspacedim = \rank(\scorematrix)$}
\\
&=
\sqrt{\scoresubspacedim} \sigmamin^{-1}(\scorematrix) \|  \lossmatrix V \|_{HS}
&&
\quad\text{//rotation invariance of $\|\cdot\|_2$}
\\
&\leq 
\sqrt{\scoresubspacedim} \sigmamin^{-1}(\scorematrix) \sqrt{\outputvarcard} \Lmax \qbound
=: D &&
\quad\text{//the definition of $\|\cdot\|_{HS}$ and triangular inequality}
\end{align*}
where $\|\cdot\|_{HS}$ and $\| \cdot \|_2$ denote the Hilbert-Schmidt and spectral norms, respectively, and $\sigmamin^{-1}(\scorematrix)$ stands for the smallest singular value of the matrix~$\scorematrix$.
The last inequality follows from the definition of the Hilbert-Schmidt norm $\|  \lossmatrix V \|_{HS}^2 = \sum_{i=1}^\outputvarcard \|\sum_{c=1}^{\outputvarcard} \lossmatrix(i, c) v_c\|_{\hilbertspace}^2$ and from the triangular inequality $\|\sum_{c=1}^{\outputvarcard} \lossmatrix(i, c) v_c\|_{\hilbertspace} \leq \sum_{c=1}^{\outputvarcard} |\lossmatrix(i, c)|  \|v_c\|_{\hilbertspace} \leq \Lmax \qbound$ thus giving $\|  \lossmatrix V \|_{HS} \leq \sqrt{\outputvarcard} \Lmax \qbound$. 

Analogously, we now bound the Hilbert-Schmidt norm of the stochastic gradient~$\gv_{\inputvarv, \outputvarv}(\parammatrix)$.
\begin{align*}
\|\gv_{\inputvarv, \outputvarv}(\parammatrix)\|_{HS}
&\leq
\frac{1}{\outputvarcard}\| \scorematrix^\transpose \scorematrix \parammatrix \featuremap(\inputvarv) + \scorematrix^\transpose \lossmatrix(:, \outputvarv)) \|_2 \|\featuremap(\inputvarv)\|_{\hilbertspace}
\\
&\leq
\frac{1}{\outputvarcard}(\| \scorematrix^\transpose \scorematrix \parammatrix \featuremap(\inputvarv)\|_2 + \|\scorematrix^\transpose \lossmatrix(:, \outputvarv)) \|_2) \|\featuremap(\inputvarv)\|_{\hilbertspace}
\\
&\leq
\frac{1}{\outputvarcard}(\| \scorematrix^\transpose \scorematrix\|_2 \|\parammatrix\|_{HS} \|\featuremap(\inputvarv)\|_{\hilbertspace} + \|\scorematrix\|_2 \|\lossmatrix(:, \outputvarv)) \|_2) \|\featuremap(\inputvarv)\|_{\hilbertspace}
\\
&\leq
\frac{1}{\outputvarcard}\sigmamax^2(\scorematrix) D \rkhsbound^2 + \frac{1}{\outputvarcard}\sigmamax(\scorematrix) \sqrt{\outputvarcard}  \Lmax \rkhsbound =: M
\end{align*}
where $\rkhsbound$ is an upper bound on $\|\featuremap(\inputvarv)\|_{\hilbertspace}$ and $\sigmamax(\scorematrix)$ is a maximal singular value of~$\scorematrix$.
Here the first inequality follows from the fact that the rank of~$\gv_{\inputvarv, \outputvarv}(\parammatrix)$ equals 1 and from submultiplicativity of the spectral norm.
We also use the inequality~$\|\parammatrix \featuremap(\inputvarv)\|_{2} \leq \|\parammatrix\|_{HS} \|\featuremap(\inputvarv)\|_{\hilbertspace}$, which follows from the properties of the Hilbert-Schmidt norm.

The bound of Theorem~\ref{th:rkhsSgdConvergence} contains the quantity $DM$ and the step size of ASGD depends on $\frac{D}{M}$, so, to be practical, both quantities cannot be exponential (for numerical stability; but the important quantity is the number of iterations from Theorem~\ref{th:calibrationSGD:kernels}).
We have
\begin{align*}
DM
& =
\condnum^2(\scorematrix) \rkhsbound^2 \scoresubspacedim \Lmax^2 \qbound^2 + \condnum(\scorematrix) \rkhsbound \sqrt{\scoresubspacedim} \Lmax^2 \qbound
=
\Lmax^2 \xi(\condnum(\scorematrix) \sqrt{\scoresubspacedim} \rkhsbound  \qbound), \quad \xi(z) = z^2 + z,
\\
\frac{M}{D}
&=\frac{\sigmamax^2(\scorematrix)}{\outputvarcard} \rkhsbound^2 + \frac{\sigmamax(\scorematrix)\sigmamin(\scorematrix)}{\outputvarcard} \frac{\rkhsbound}{\qbound\sqrt{\scoresubspacedim}}
\end{align*}
where $\condnum(\scorematrix) = \frac{\sigmamax}{\sigmamin}$ is the condition number of $\scorematrix$.
Note that the quantity~$DM$ is invariant to the scaling of the matrix~$\scorematrix$.
The quantity~$\frac{D}{M}$ scales proportionally to the square of the scale of~$\scorematrix$ and thus rescaling~$\scorematrix$ can always bring it to~$\bigO(1)$. For the rest of the analysis, we consider $\rkhsbound$ and $\qbound$ to be well-behaved constants and thus focus on the dependence of the quantity~$DM$ on $\scorematrix$ and $\lossmatrix$. 

\subsection{Constants for specific losses}
\label{sec:tranferAndSgd:constants}
We now estimate the product~$DM$ from~\eqref{eq:sgdConstant} for the 0-1, block 0-1 and Hamming losses.
For the definition of the losses and the corresponding matrices~$\scorematrix$, we refer to Section~\ref{sec:bounds:interpretations}.

\textbf{0-1 loss.}
For the 0-1 loss~$\lossmatrix_{01}$ and $\scorematrix = \id_\outputvarcard$, we have $\Lmax=1$, $\scoresubspacedim=\outputvarcard$, $\sigmamin=\sigmamax=1$, thus $DM = \bigO(\outputvarcard)$ is very large leading to very slow convergence of ASGD.

\textbf{Block 0-1 loss.}
For the block 0-1 loss~$\lossmatrix_{01,\numblocks}$ and matrix~$\scorematrix_{01,\numblocks}$, we have $\Lmax=1$, $\scoresubspacedim=\numblocks$, $\sigmamin=\sigmamax=\sqrt{\blocksize}$, thus $DM = \bigO(\numblocks)$.

\textbf{Hamming loss.} For the Hamming loss, we have $\Lmax=1$, $\scoresubspacedim=\log_2 \outputvarcard + 1$, $\condnum(\scorematrix_{\hamming,\hamminglen}) \leq \log_2{\outputvarcard}+2$ (see the derivation in Section~\ref{sec:hammingLossProps}). Finally, we have $DM = \bigO(\log_2^3 \outputvarcard)$.

\section{Properties of the basis of the Hamming loss}
\label{sec:hammingLossProps}
As defined in~\eqref{eq:hammingLoss}, the matrix~$\lossmatrix_{\hamming,\hamminglen} \in \R^{\outputvarcard \times \outputvarcard}$ is the matrix of the Hamming loss between tuples of $\hamminglen$~binary variables, and the number of labels equals $\outputvarcard = 2^\hamminglen$.
Also recall that
$
\scorematrix_{\hamming,\hamminglen} := [ \frac12 \one_{2^\hamminglen}, \hv^{(1)}, \dots, \hv^{(\hamminglen)} ],
$
$
(\hv^{(\hammingindex)})_{\hat{\outputvarv}} := [\hat{\outputvar}_\hammingindex = 1],
$
$\hammingindex = 1,\dots,\hamminglen$.
We have $
\scoresubset_{\hamming,\hamminglen}
=
\colspace(\scorematrix_{\hamming,\hamminglen})
=
\colspace(\lossmatrix_{\hamming,\hamminglen})
$
 and $\rank(\lossmatrix_{\hamming,\hamminglen}) \!=\! \rank(\scorematrix_{\hamming,\hamminglen}) \!=\! \hamminglen\!+\!1$.

We now explicitly compute $\max_{i\neq j}\|\proj_{\scoresubset_{\hamming,\hamminglen}} \Delta_{ij}\|_2^2$.
We shortcut~$\scorematrix_{\hamming,\hamminglen}$ by~$\scorematrix$ and compute
\begin{equation}
\label{eq:hammingLoss:gramMatrix}
\scorematrix^\transpose \scorematrix = 2^{\hamminglen-2} \begin{bmatrix}
1 & 1 & \cdots  & 1 \\
1 & 2 & 1 & \cdots \\
1 & 1 & 2 & \cdots \\
\cdots & \cdots & \cdots & 1\\
1 & \cdots& 1  & 2
\end{bmatrix}.
\end{equation}
We can compute the inverse matrix explicitly as well:
\begin{equation}
\label{eq:hammingLoss:gramMatrixInv}
(\scorematrix^\transpose \scorematrix)^{-1} = 2^{2-\hamminglen} \begin{bmatrix}
1+\hamminglen & -1 & \cdots  & -1 \\
-1 & 1 & 0 & \cdots \\
-1 & 0 & 1 & \cdots \\
\cdots & \cdots & \cdots & 0\\
-1 & \cdots & 0  & 1
\end{bmatrix}.
\end{equation}
The vector $\scorematrix^\transpose \Delta_{ij}$ equals the difference of the two rows of~$\scorematrix$, i.e., $[0, c_1, \dots, c_\hamminglen]^\transpose \in \R^{\hamminglen+1}$ with each~$c_\hammingindex \in \{-1, 0, +1\}$. We explicitly compute the square norm~$\|\proj_{\scoresubset_{\hamming,\hamminglen}} \Delta_{ij}\|_2^2$: 
\[
\|\proj_{\scoresubset_{\hamming,\hamminglen}} \Delta_{ij}\|_2^2
=
\Delta_{ij}^\transpose \scorematrix (\scorematrix^\transpose \scorematrix)^{-1} \scorematrix^\transpose \Delta_{ij}
=
[0, c_1, \dots, c_\hamminglen] (\scorematrix^\transpose \scorematrix)^{-1} [0, c_1, \dots, c_\hamminglen]^\transpose
= 2^{2-\hamminglen} \sum_{t=1}^\hamminglen c_t^2,
\]
where the last equality follows from the identity submatrix of~\eqref{eq:hammingLoss:gramMatrixInv} and from the zero in the first position of the vector~$\scorematrix^\transpose \Delta_{ij}$.
The quantity $\|\proj_{\scoresubset_{\hamming,\hamminglen}} \Delta_{ij}\|_2^2$ is maximized when none of~$c_t$ equals zero, which is achievable, e.g., when the label~$i$ corresponds to all zeros and the label~$j$ to all ones.
We now have $\max_{i\neq j}\|\proj_{\scoresubset_{\hamming,\hamminglen}} \Delta_{ij}\|_2^2 = \frac{4\hamminglen}{2^\hamminglen}$.

We now compute the smallest and largest eigenvalues of the Gram matrix~\eqref{eq:hammingLoss:gramMatrix} for $\scorematrix_{\hamming,\hamminglen}$.
Ignoring the scaling factor~$2^{\hamminglen-2}$, we see by Gaussian elimination that the determinant and thus the product of all eigenvalues equals~$1$.
If we subtract~$\id_{\hamminglen+1}$ the matrix becomes of rank~$2$, meaning that $\hamminglen-1$ eigenvalues equal~$1$.
The trace, i.e., the sum of the eigenvalues of~\eqref{eq:hammingLoss:gramMatrix}, without the scaling factor~$2^{\hamminglen-2}$ equals~$2\hamminglen+1$.
Summing up, we have~$\lambdamin \lambdamax = 1$ and $\lambdamin + \lambdamax = \hamminglen + 2$.
We can now compute~$\lambdamin = \frac12(\hamminglen+2 - \sqrt{\hamminglen^2 + 4\hamminglen}) \in [\frac{1}{\hamminglen+2}, \frac{1}{\hamminglen}]$ and~$\lambdamax = \frac12(\hamminglen+2 + \sqrt{\hamminglen^2 + 4\hamminglen}) \in [\hamminglen+1, \hamminglen+2]$.
By putting back the multiplicative factor, we get 
$
\sigmamin = \sqrt{\lambdamin} \geq \frac{\sqrt{\outputvarcard}}{2\sqrt{\log_2{\outputvarcard}+2}}
$
and
$
\sigmamax = \sqrt{\lambdamax} \leq \frac{\sqrt{\outputvarcard}}{2}\sqrt{\log_2{\outputvarcard}+2}
$,
and thus the condition number is~$\condnum \leq \log_2{\outputvarcard}+2$.

\end{document}